\documentclass[twoside]{article}

\usepackage{amssymb}
\usepackage{algorithm}
\usepackage{algorithmic}
\usepackage{booktabs}
\usepackage{xcolor} 
\usepackage{ulem}
\usepackage{graphicx} 
\usepackage{amsthm}
\usepackage{url}
\newtheorem{proposition}{Proposition}
\newtheorem{corollary}{Corollary}

\usepackage[preprint]{aistats2027}

\usepackage[round]{natbib}

\begin{document}

\twocolumn[

\aistatstitle{%
Frequency-Decoupled Diffusion Guidance for Non-Blind Image Deblurring}

\aistatsauthor{ Sihan Wang$^{1,\dagger}$ \And Jinshu Huang$^{1,\dagger,*}$ \And  Haibin Su$^{2}$ \And Yunhua Xue$^{1}$}

\aistatsaddress{$^{1}$School of Mathematical Sciences, Nankai University, Tianjin, China. \\
$^{2}$Yau Mathematical Sciences Center, Tsinghua University, Beijing, China.   \\

} 
]
\begingroup
\renewcommand{\thefootnote}{\ensuremath{\dagger}}
\footnotetext{Equal contribution (wangsihan@mail.nankai.edu.cn; huangjsh@mail.nankai.edu.cn).}
\endgroup
\begingroup
\renewcommand{\thefootnote}{*}
\footnotetext{Corresponding author.}
\endgroup

\begin{abstract}
 Pretrained diffusion models provide powerful image priors for training-free posterior sampling in image restoration. To guide this sampling process, frequency-aware methods progressively incorporate measurement information across frequency bands, facilitating coarse-to-fine reconstruction. 
However, existing methods typically do not explicitly separate frequency activation from degradation-induced attenuation, leaving attenuation differences among inactive frequencies insufficiently modeled.
In this work, we propose frequency-decoupled posterior guidance to separate frequency activation from attenuation-aware spectral regularization. 
Specifically, a progressive low-to-high frequency schedule determines the active measurement band, while a kernel-derived attenuation map defines a selective spectral prior over inactive components. 
To stabilize the sampling process, we also introduce a local trajectory regularizer that suppresses spatially irregular state-to-clean deviations.
For a fixed endpoint energy, we provide a KL-regularized path-space interpretation. In practice, we construct time-dependent guidance through local energy corrections using a Tweedie plug-in approximation. Experiments on natural-image benchmarks demonstrate strong PSNR and SSIM performance across challenging non-blind deblurring settings, even at higher measurement noise levels. 
\end{abstract}

\section{Introduction}
\label{sec:introduction}
Image restoration aims to recover a clean image from incomplete or
corrupted measurements. A broad class of restoration tasks takes the form
\begin{equation}
    \mathrm{y}=\mathrm{A}\mathrm{x}_0+\mathrm{n},
    \qquad
    \mathrm{n}\sim\mathcal{N}
    \bigl(\mathrm{0},\sigma_{\mathrm{y}}^2\mathrm{I}_m\bigr),
    \label{eq:model}
\end{equation}
where $\mathrm{x}_0\in\mathbb{R}^{d}$ is the unknown image,
$\mathrm{y}\in\mathbb{R}^{m}$ is the observation, and
$\mathrm{A}:\mathbb{R}^{d}\to\mathbb{R}^{m}$ is a known linear operator.
Recovering $\mathbf{x}_0$ is generally ill-posed because operator $\mathrm{A}$ may attenuate or discard part of the image information. Effective restoration must therefore balance measurement fidelity with image priors.

Classical methods design the prior through handcrafted regularizers, while more recent approaches use learned denoisers or energy models~\citep{RUDIN1992259,cai2012image,glowinski2017splitting,su2026, roth2009fields,venkatakrishnan2013plug,dong2015image}. Diffusion models provide an expressive alternative by learning the score fields of progressively perturbed data distributions~\citep{ho2020ddpm,song2021scorebased,huang2026score}. Conditional diffusion models can learn a direct degraded-to-clean mapping, but typically require paired data and task-specific training~\citep{saharia2022palette,whang2022deblurring}. Zero-shot methods instead combine a pretrained unconditional diffusion model with the known forward operator at inference time~\citep{kawar2022ddrm,song2023pseudoinverse,wang2023ddnm,chung2023dps,zhu2023diffpir}. This decoupling of the image prior from the observation model enables one pretrained model to address multiple inverse problems without retraining.

Despite this flexibility, existing zero-shot methods commonly enforce measurement consistency through spatial residuals, pseudoinverse projections, or likelihood gradients~\citep{kawar2022ddrm,wang2023ddnm,chung2023dps,meng2025dmps}. For natural images, diffusion denoising often recovers coarse structure before fine detail~\citep{falck2025fourier}. This behavior has motivated frequency-aware guidance, which progressively incorporates measurement information across frequency bands. FGPS follows this principle by gradually exposing measurement frequencies~\citep{thaker2025fgps}; noise-conditioned posterior continuation further couples frequency activation to the diffusion noise level~\citep{tian2026stabilizing}. However, existing frequency-aware methods typically do not explicitly separate frequency activation from degradation-induced attenuation, leaving attenuation differences among inactive frequencies insufficiently modeled. This distinction matters for directional or frequency-selective blur, where frequencies at similar spatial scales can have substantially different responses.

Motivated by these observations, we propose frequency-decoupled posterior guidance for non-blind image deblurring. By decoupling frequency activation from attenuation-aware regularization, our approach enables more precise incorporation of measurement information, effectively preventing highly attenuated or noisy frequencies from corrupting the update direction. Meanwhile, it better models inactive components by preventing the diffusion prior from introducing hallucination artifacts in spectral regions where measurement data is absent or severely degraded. 
Additionally, we introduce a local trajectory-stabilization regularizer that penalizes spatially irregular state-to-clean deviations, thereby promoting stable and spatially coherent sampling.
Our main contributions are summarized as follows:
\begin{itemize}
\item We introduce frequency-decoupled diffusion guidance for non-blind image deblurring. It is realized via a progressive schedule that determines which measurement frequencies are active at each reverse step, along with a kernel-derived attenuation map that defines a selective spectral prior over inactive components.

\item 
 We formulate these components in a time-dependent guidance energy and derive practical stepwise corrections using a Tweedie plug-in approximation. To stabilize the updates, we introduce a local trajectory regularizer that suppresses spatially irregular state-to-clean deviations. For a fixed endpoint energy, we also provide a KL/Gibbs path-space interpretation.

\item We evaluate the proposed framework on natural-image benchmarks across challenging non-blind deblurring settings. The method achieves strong PSNR and SSIM performance and remains robust as measurement noise increases. Controlled ablations clarify the roles of the individual components and identify trajectory stabilization as a key contributor to the observed improvements.
\end{itemize}

\section{Background}
\label{sec:background}

\subsection{Notation}
\label{sec:notation}

We use $\odot$ for element-wise multiplication, $\mathrm{I}_d$ for the
$d$-dimensional identity matrix, and $\mathbf{1}_{\Lambda}$ for the
indicator function of a set $\Lambda$. The discrete Fourier transform
(DFT) is denoted by $\mathcal{F}$, with $\omega$ indexing spatial
frequencies. We write $\mathrm{X}_t$ for a random state and
$\mathrm{x}_t$ for its realization. For $\mathrm{z}\in\mathbb{C}^{N}$,
define
$\|\mathrm{z}\|_{q,\mathrm{avg}}
:=(N^{-1}\sum_{j=1}^{N}|\mathrm{z}_j|^q)^{1/q}$
for $q\in\{1,2\}$.

\subsection{Diffusion Models}
\label{sec:diffusion_background}

Diffusion models progressively perturb clean data with Gaussian noise
and learn the corresponding reverse dynamics
~\citep{ho2020ddpm,song2021scorebased}. For scalar functions $f$ and $g$,
a continuous-time forward process can be written as
\begin{equation}
    \mathrm{d}\mathrm{X}_t
    =
    f(t)\mathrm{X}_t\,\mathrm{d}t
    +
    g(t)\,\mathrm{d}\mathrm{B}_t,
    \qquad
    \mathrm{X}_0\sim p_{\mathrm{data}},
    \label{eq:forward_sde}
\end{equation}
where $\mathrm{B}_t$ is a standard $d$-dimensional Wiener process.
Let $p_t$ denote the marginal density of $\mathrm{X}_t$. With time
running backward from $T$ to $0$, the reverse-time SDE is
\begin{equation}
    \mathrm{d}\mathrm{X}_t
    =
    \left[
        f(t)\mathrm{X}_t
        -
        g(t)^2\nabla\log p_t(\mathrm{X}_t)
    \right]\mathrm{d}t
    +
    g(t)\,\mathrm{d}\bar{\mathrm{B}}_t,
    \label{eq:reverse_sde}
\end{equation}
where $\bar{\mathrm{B}}_t$ is a reverse-time Wiener process and
$\nabla\log p_t$ denotes the marginal score. Replacing this score by a
learned model $\mathrm{s}_{\Theta}(t,\mathrm{x}_t)$ and discretizing the
reverse dynamics yields a practical sampler.

For the variance-preserving perturbation used in DDPMs
~\citep{ho2020ddpm}, let
$\bar{\alpha}_t=\prod_{j=1}^{t}(1-\beta_j)$, where
$\beta_j\in(0,1)$ is the discrete noise schedule. The forward
perturbation at step $t$ is
\begin{equation}
    \mathrm{X}_t
    =
    \sqrt{\bar{\alpha}_t}\,\mathrm{X}_0
    +
    \sqrt{1-\bar{\alpha}_t}\,
    \mathrm{\epsilon},
    \qquad
    \mathrm{\epsilon}
    \sim\mathcal{N}(\mathrm{0},\mathrm{I}_d),
    \label{eq:ddpm_forward}
\end{equation}
with $\mathrm{\epsilon}$ independent of $\mathrm{X}_0$.
Substituting the learned score into Tweedie's formula~\citep{efron2011tweedie} gives
\begin{equation}
    \widehat{\mathrm{x}}_{0,t}
    =
    \frac{
        \mathrm{x}_t
        +
        (1-\bar{\alpha}_t)
        \mathrm{s}_{\Theta}(t,\mathrm{x}_t)
    }{
        \sqrt{\bar{\alpha}_t}
    },
    \label{eq:tweedie_background}
\end{equation}
which equals
$\mathbb{E}[\mathrm{X}_0\mid\mathrm{X}_t=\mathrm{x}_t]$
when the learned score equals the true marginal score.

\subsection{Diffusion Posterior Guidance for Inverse Problems}
\label{sec:inverse_diffusion_background}

Conditioned on the observation $\mathrm{y}$, the conditional score can be decomposed as
\begin{equation}
\begin{aligned}
    \nabla_{\mathrm{x}_t}
    \log p_t(\mathrm{x}_t\mid\mathrm{y})
    &=
    \nabla_{\mathrm{x}_t}
    \log p_t(\mathrm{x}_t)
    \\
    &\quad+
    \nabla_{\mathrm{x}_t}
    \log p_t(\mathrm{y}\mid\mathrm{x}_t).
\end{aligned}
\label{eq:conditional_score_decomposition}
\end{equation}
The first term is approximated by the pretrained score model.
Because the observation depends on $\mathrm{X}_t$ through
$\mathrm{X}_0$, the second term involves
\begin{equation}
    p_t(\mathrm{y}\mid\mathrm{x}_t)
    =
    \int
    p(\mathrm{y}\mid\mathrm{x}_0)
    p(\mathrm{x}_0\mid\mathrm{x}_t)
    \,\mathrm{d}\mathrm{x}_0,
    \label{eq:marginal_likelihood}
\end{equation}
which is generally intractable. Diffusion Posterior Sampling (DPS)~\citep{chung2023dps} uses the plug-in approximation
$p_t(\mathrm{y}\mid\mathrm{x}_t)
\approx p(\mathrm{y}\mid\widehat{\mathrm{x}}_{0,t})$.
For the Gaussian observation model in \eqref{eq:model}, its
corresponding guidance surrogate is
\begin{equation}
    \nabla_{\mathrm{x}_t}
    \log p_t(\mathrm{y}\mid\mathrm{x}_t)
    \approx
    -
    \frac{1}{2\sigma_{\mathrm{y}}^2}
    \nabla_{\mathrm{x}_t}
    \left\|
        \mathrm{A}\widehat{\mathrm{x}}_{0,t}
        -
        \mathrm{y}
    \right\|_2^2.
    \label{eq:dps_guidance}
\end{equation}
Thus, measurement guidance can be computed by differentiating through
the clean-image estimate while keeping the pretrained score model
fixed. Our method applies this plug-in principle to a time-dependent,
frequency-decoupled energy.

\section{Related Work}

\subsection{Training-Based Diffusion Models for Image Restoration}

Conditional diffusion models learn restoration distributions from paired degraded and clean images. SR3~\citep{saharia2023sr3} conditions the reverse process on a low-resolution observation for image super-resolution, while Palette~\citep{saharia2022palette} extends conditional diffusion to a broader family of image-to-image tasks. I$^2$SB~\citep{liu2023i2sb}, IR-SDE~\citep{luo2023irsde}, and ResShift~\citep{yue2023resshift} further incorporate the degradation process into the diffusion bridge or stochastic restoration dynamics. Their learned models or conditioning mechanisms are typically tied to the degradation distributions used in training.

\subsection{Training-Free Diffusion Methods for Inverse Problems}
A common training-free strategy keeps the pretrained diffusion model fixed and uses the known forward operator to guide sampling.

\textbf{Operator information and measurement guidance.}
DDRM~\citep{kawar2022ddrm} exploits the singular-vector structure of linear operators, while DDNM~\citep{wang2023ddnm} uses a range--null-space decomposition. DPS~\citep{chung2023dps} evaluates a likelihood surrogate at a denoised clean-image estimate; DMPS~\citep{meng2025dmps} instead derives a closed-form likelihood-score approximation that avoids back-propagation through the denoiser.

\textbf{Frequency exposure and stabilization.}
A related line of work controls which frequency components are exposed to measurement consistency during sampling. 
FGPS~\citep{thaker2025fgps} progressively exposes measurement frequencies during
sampling.
Noise-conditioned posterior continuation~\citep{tian2026stabilizing} combines frequency exposure
with a Haar-domain rule for committing coarse and detail corrections.
MASF~\citep{qian2024masf} stabilizes generation by averaging frequency components of
clean-image predictions across sampling steps.
Our method uses kernel-derived attenuation to weight a spectral prior over inactive frequencies and regularizes state-to-clean deviations locally.

\textbf{Variational and reconstruction-oriented methods.}
Diffusion priors can also be incorporated into iterative reconstruction and explicit optimization objectives.
DiffPIR~\citep{zhu2023diffpir} combines diffusion denoising with plug-and-play restoration, while RED-Diff~\citep{mardani2023variational} uses variational approximation and denoising regularization. DMAP~\citep{xu2025rethinking} reinterprets DPS through posterior maximization and combines multiple likelihood-ascent steps with projection onto the reverse-transition Gaussian shell to improve reconstruction performance.

\section{Method}
\label{sec:method}
We propose a frequency-decoupled guidance framework for non-blind image deblurring. We first introduce the spectral guidance energy, then give its variational interpretation and describe the practical sampling rule.

\subsection{Frequency-Decoupled Guidance Energy}
\label{subsec:frequency_decoupled_energy}

To make the spectral construction explicit, consider non-blind deblurring with a known blur kernel $\kappa$ and observation operator $\mathrm{A}:\mathbb{R}^{d}\to\mathbb{R}^{d}$.
For a convolutional realization of the degradation,
\begin{equation}
	\mathrm{A}\mathrm{x}
	=
	\kappa * \mathrm{x},
   \label{eq:spatially_invariant_blur}
\end{equation}
where $*$ denotes discrete convolution under a specified boundary convention.

When this boundary convention diagonalizes the convolution in the discrete Fourier transform (DFT), the observation model becomes
\begin{equation}
	\bigl[\mathcal{F}(\mathrm{y})\bigr](\omega)
	=
	\widehat{\kappa}(\omega)
	\bigl[\mathcal{F}(\mathrm{x}_0)\bigr](\omega)
	+
	\bigl[\mathcal{F}(\mathrm{n})\bigr](\omega),
    \nonumber
\end{equation}
where $\widehat{\kappa}(\omega)$ is the DFT of $\kappa$. For boundary conventions that do not yield exact DFT diagonalization, we use $\widehat{\kappa}$ only to construct the kernel-based attenuation weight below.

\paragraph{Time-dependent frequency scheduling.}
We define the kernel-dependent spectral attenuation weight
\begin{equation}
    W_{\kappa}(\omega)
    =
    \frac{
        \rho
    }{
        |\widehat{\kappa}(\omega)|^2+\rho
    },
    \qquad \rho>0.
    \label{eq:kernel_attenuation_weight}
\end{equation}
The weight $W_{\kappa}(\omega)$ measures attenuation induced by the blur
kernel, which is small when the kernel retains a strong response and approaches one when the corresponding frequency is strongly attenuated.

Diffusion sampling typically establishes coarse structures before fine details. We therefore use a progressive low-to-high schedule for measurement consistency. At reverse step $t$, define
\begin{equation}
	\mathcal{B}_t
	=
	\left\{
	\omega:
	\|\omega\|_2<\tau_t
	\right\},
	\label{eq:active_frequency_set}
\end{equation}
and the corresponding binary mask 
\begin{equation}
	G_t(\omega)
	=
	\mathbf{1}_{\mathcal{B}_t}(\omega).
	\label{eq:lowpass_mask}
\end{equation}
The cutoffs satisfy
\begin{equation}
    \tau_T
    \leq
    \tau_{T-1}
    \leq
    \cdots
    \leq
    \tau_1,
\end{equation}
so the support of $G_t$ expands as sampling approaches the clean image. %
To identify strongly attenuated frequencies outside the active band, we further define
\begin{equation}
    M_t(\omega)
    =
    \bigl(1-G_t(\omega)\bigr)
    W_{\kappa}(\omega).
    \label{eq:spectral_prior_mask}
\end{equation}
{Thus, $M_t$ vanishes on the active band and assigns larger weights to inactive frequencies with stronger kernel attenuation.}

\paragraph{Frequency-decoupled guidance energy.}
Let
\begin{equation}
    R_{\mathrm{y}}(\mathrm{x})
    :=
    \mathcal{F}
    \bigl(
        \mathrm{A}\mathrm{x}-\mathrm{y}
    \bigr)
    \nonumber
\end{equation}
denote the Fourier-domain measurement residual. Using the frequency
weights above, we define
\begin{equation}
\begin{aligned}
    \Phi_{t,\mathrm{y}}(\mathrm{x})
    ={}&
    \lambda_{\mathrm{lp},t}
    \left\|
        G_t\odot R_{\mathrm{y}}(\mathrm{x})
    \right\|_{2,\mathrm{avg}}^2
    \\
    &+
    \lambda_{\mathrm{prior},t}
    \left\|
        M_t\odot\mathcal{F}(\mathrm{x})
    \right\|_{1,\mathrm{avg}},
\end{aligned}
\label{eq:frequency_decoupled_energy}
\end{equation}
where $\lambda_{\mathrm{lp},t}\geq0$ and
$\lambda_{\mathrm{prior},t}\geq0$. %
The first term imposes measurement consistency on the active band, whereas the second suppresses strongly attenuated components that have not yet entered this band. As $G_t$ expands, the corresponding frequencies leave $M_t$ and are handled by measurement consistency. During sampling, we evaluate this energy at the Tweedie estimate $\widehat{\mathrm{x}}_{0,t}$.

\subsection{Variational Interpretation of Energy Guidance}
\label{sec:posterior_regularization}
We begin with a fixed endpoint energy, for which the induced score correction has an exact variational form.
Let $P_{\Theta}$ denote the joint path distribution of the pretrained reverse trajectory
$\mathrm{X}_{0:T}=(\mathrm{X}_0,\ldots,\mathrm{X}_T)$, starting from $\mathrm{X}_T\sim\mathcal{N}(\mathrm{0},\mathrm{I}_d)$. For a fixed observation $\mathrm{y}$, let
\begin{equation}
    \Phi_{\mathrm{y}}:\mathbb{R}^d\rightarrow[0,\infty)
\end{equation}
be a measurable endpoint energy \footnote{Here, ``fixed'' means that $\Phi_{\mathrm{y}}$ has no explicit dependence on the diffusion step $t$, although it may depend on both $\mathrm{x}_0$ and the fixed observation $\mathrm{y}$}. Define
\begin{equation}
    Z_{\mathrm{y}}
    :=
    \mathbb{E}_{P_{\Theta}}
    \left[
        \exp\bigl(-\Phi_{\mathrm{y}}(\mathrm{X}_0)\bigr)
    \right].
\end{equation}
Then $0<Z_{\mathrm{y}}\leq 1$.

We consider the minimization over trajectory distributions absolutely continuous with respect to $P_{\Theta}$:
\begin{equation}
    P_{\mathrm{y}}^{*}
    =
    \underset{P\ll P_{\Theta}}{\arg\min}
    \left\{
        \operatorname{KL}\bigl(P\,\|\,P_{\Theta}\bigr)
        +
        \mathbb{E}_{P}
        \left[
            \Phi_{\mathrm{y}}(\mathrm{X}_0)
        \right]
    \right\}.
    \label{eq:path_regularization_short}
\end{equation}
By the Gibbs variational principle, the unique minimizer is
\begin{equation}
    \frac{\mathrm{d}P_{\mathrm{y}}^{*}}
         {\mathrm{d}P_{\Theta}}
    (\mathrm{x}_{0:T})
    =
    \frac{
        \exp\bigl[-\Phi_{\mathrm{y}}(\mathrm{x}_0)\bigr]
    }{
        Z_{\mathrm{y}}
    }.
    \label{eq:gibbs_reweighting}
\end{equation}
Thus, the endpoint energy reweights trajectories of the pretrained reverse process; see Supplementary Sec.~\ref{app:proof_path_regularization} for the proof and \citet{uehara2025finetuning} for related entropy-regularized diffusion formulations.

For $t\in\{1,\ldots,T\}$, suppose that the marginal law of $\mathrm{X}_t$ and the conditional law of $\mathrm{X}_0$ given $\mathrm{X}_t=\mathrm{x}_t$ have densities with respect to the Lebesgue measure. Denote them by $p_{\Theta,t}(\mathrm{x}_t)$ and
$p_{\Theta}(\mathrm{x}_0\mid\mathrm{x}_t)$, respectively.
The marginal density of $\mathrm{X}_t$ under the energy-tilted path measure $P_{\mathrm{y}}^*$ is then
\begin{equation}
    \pi_t^{\Phi}(\mathrm{x}_t;\mathrm{y})
    =
    \frac{1}{Z_{\mathrm{y}}}
    p_{\Theta,t}(\mathrm{x}_t)
    Z_t(\mathrm{x}_t;\mathrm{y}),
    \label{eq:tilted_marginal}
\end{equation}
where
\begin{equation}
\begin{aligned}
    Z_t(\mathrm{x}_t;\mathrm{y})
    &:=
    \mathbb{E}_{P_{\Theta}}
    \left[
        \exp\bigl(-\Phi_{\mathrm{y}}(\mathrm{X}_0)\bigr)
        \,\middle|\,
        \mathrm{X}_t=\mathrm{x}_t
    \right]
    \\
    &=
    \int
    p_{\Theta}(\mathrm{x}_0\mid\mathrm{x}_t)
    \exp\bigl[-\Phi_{\mathrm{y}}(\mathrm{x}_0)\bigr]
    \,\mathrm{d}\mathrm{x}_0.
\end{aligned}
\end{equation}
At the endpoint, {assuming that the marginal law of $\mathrm{X}_0$ has density $p_{\Theta,0}$}, Eq.~\eqref{eq:gibbs_reweighting} directly gives
\begin{equation}
    \pi_0^{\Phi}(\mathrm{x}_0;\mathrm{y})
    =
    \frac{1}{Z_{\mathrm{y}}}
    p_{\Theta,0}(\mathrm{x}_0)
    \exp
    \left[
        -\Phi_{\mathrm{y}}(\mathrm{x}_0)
    \right].
    \label{eq:tilted_endpoint}
\end{equation}
{If $p_{\Theta,t}(\mathrm{x}_t)$ and $Z_t(\mathrm{x}_t;\mathrm{y})$ are positive and differentiable in $\mathrm{x}_t$, taking the score gives}
\begin{equation}
\begin{aligned}
    \nabla_{\mathrm{x}_t}
    \log\pi_t^{\Phi}(\mathrm{x}_t;\mathrm{y})
    ={}&
    \nabla_{\mathrm{x}_t}
    \log p_{\Theta,t}(\mathrm{x}_t)
    \\
    &+
    \nabla_{\mathrm{x}_t}
    \log Z_t(\mathrm{x}_t;\mathrm{y}).
\end{aligned}
\label{eq:score_decomposition_short}
\end{equation}
{The first term is the score of the reference marginal induced by the pretrained reverse process}, while the second is the correction induced by the endpoint energy.

For the Gaussian observation model in Eq.~\eqref{eq:model},
\begin{equation}
    -\log p(\mathrm{y}\mid\mathrm{x}_0)
    =
    \frac{1}{2\sigma_{\mathrm{y}}^2}
    \left\|
        \mathrm{A}\mathrm{x}_0-\mathrm{y}
    \right\|_2^2
    +
    C_{\mathrm{y}},
    \label{eq:gaussian_negative_log_likelihood}
\end{equation}
where $C_{\mathrm{y}}$ is independent of $\mathrm{x}_0$. Hence, if $\Phi_{\mathrm{y}}$ equals the negative log-likelihood up to
an additive constant independent of $\mathrm{x}_0$, then
\begin{equation}
    \pi_0^{\Phi}(\mathrm{x}_0;\mathrm{y})
    \propto
    p_{\Theta,0}(\mathrm{x}_0)
    p(\mathrm{y}\mid\mathrm{x}_0),
\end{equation}
which is the Bayesian posterior under the pretrained diffusion prior. When the energy contains the frequency-dependent prior or other terms
not implied by the physical likelihood, the resulting distribution is instead an energy-tilted, or generalized, posterior.

{In our method, $\Phi_{t,\mathrm{y}}$ varies with $t$, so the fixed-energy Gibbs construction above does not generally define a single path measure shared across all reverse steps. We therefore apply this construction locally: at each step $t$, we treat $\Phi_{t,\mathrm{y}}$ as fixed and approximate the resulting score correction as described in the next subsection.}

\subsection{Practical Energy Guidance and Stable Sampling}
\label{sec:practical_sampling}
We now derive the practical objective used to correct each reverse step and introduce a spatial regularizer for trajectory stabilization.

For the current step-dependent energy $\Phi_{t,\mathrm{y}}$, the local
fixed-energy construction above involves the conditional partition
function
\begin{equation}
    Z_t^{\mathrm{loc}}(\mathrm{x}_t;\mathrm{y})
    =
    \int
    p_{\Theta}(\mathrm{x}_0\mid\mathrm{x}_t)
    \exp\left[-\Phi_{t,\mathrm{y}}(\mathrm{x}_0)\right]
    \,\mathrm{d}\mathrm{x}_0.
    \label{eq:local_partition}
\end{equation}
Because the exact conditional density
$p_{\Theta}(\mathrm{x}_0\mid\mathrm{x}_t)$ is not explicitly available,
we approximate this quantity using the Tweedie estimate
$\widehat{\mathrm{x}}_{0,t}$ defined in
Eq.~\eqref{eq:tweedie_background}:
\begin{equation}
    Z_t^{\mathrm{loc}}(\mathrm{x}_t;\mathrm{y})
    \approx
    \exp\left[
        -\Phi_{t,\mathrm{y}}
        (\widehat{\mathrm{x}}_{0,t})
    \right].
    \label{eq:mean_substitution}
\end{equation}

Using the measurement residual $R_{\mathrm{y}}$ defined above, the
measurement term under this substitution is
\begin{equation}
    \mathcal{L}_{\mathrm{lp},t}
    =
    \left\|
        G_t\odot
        R_{\mathrm{y}}(\widehat{\mathrm{x}}_{0,t})
    \right\|_{2,\mathrm{avg}}^2.
    \label{eq:practical_measurement_term}
\end{equation}
Similarly, the spectral prior in
Eq.~\eqref{eq:frequency_decoupled_energy} becomes
\begin{equation}
    \mathcal{P}_t
    =
    \left\|
        M_t\odot
        \mathcal{F}(\widehat{\mathrm{x}}_{0,t})
    \right\|_{1,\mathrm{avg}}.
    \label{eq:clean_spectral_prior}
\end{equation}
The resulting practical frequency objective is
\begin{equation}
    \mathcal{L}_t^{\mathrm{freq}}
    :=
    \Phi_{t,\mathrm{y}}
    (\widehat{\mathrm{x}}_{0,t})
    =
    \lambda_{\mathrm{lp},t}
    \mathcal{L}_{\mathrm{lp},t}
    +
    \lambda_{\mathrm{prior},t}
    \mathcal{P}_t.
    \label{eq:practical_frequency_objective}
\end{equation}

\paragraph{Local trajectory stabilization and practical objective.}
The frequency objective does not directly control the spatial regularity of the state-to-clean deviation. We further introduce
\begin{equation}
    \mathcal{R}_t
    =
    \operatorname{TV}_{\mathrm{ani}}
    \left(
        \mathrm{x}_t
        -
        \operatorname{sg}
        (\widehat{\mathrm{x}}_{0,t})
    \right),
    \label{eq:trajectory_regularization}
\end{equation}
where $\operatorname{sg}(\cdot)$ denotes stop-gradient. This term discourages rapid spatial variation in the state-to-clean deviation without directly regularizing the reconstructed image or changing the spectral masks. Hence, the complete practical objective is
\begin{equation}
    \mathcal{L}_t
    =
    \mathcal{L}_t^{\mathrm{freq}}
    +
    \lambda_{\mathrm{reg},t}
    \mathcal{R}_t,
    \qquad
    \lambda_{\mathrm{reg},t}\geq0.
    \label{eq:sampling_objective}
\end{equation}

We parameterize the time-dependent coefficients and cutoff by the
normalized reverse progress
\begin{equation}
    p_t
    :=
    \frac{T-t}{T-1},
    \qquad
    t=T,T-1,\ldots,1.
    \label{eq:reverse_progress}
\end{equation}
The coefficients
$\lambda_{j,t}$, $j\in\{\mathrm{lp},\mathrm{prior},\mathrm{reg}\}$,
use fixed sigmoid gates, while the cutoff follows a monotone exponential schedule. The explicit forms used in our experiments are specified in Supplementary Sec.~\ref{app:experimental_details}.

\paragraph{Normalized guidance update.}
We now describe how the practical objective is incorporated into each
reverse step. Let $\mathcal{S}_{\Theta,t}^{(a)}$ denote the reference reverse sampler, where $a\in\{\mathrm{DDPM},\mathrm{DDIM}\}$. At step
$t$, it produces the Tweedie estimate and a reference proposal:
\begin{equation}
    \left(
        \widehat{\mathrm{x}}_{0,t},
        \widetilde{\mathrm{x}}_{t-1}
    \right)
    =
    \mathcal{S}_{\Theta,t}^{(a)}
    \left(
        \mathrm{x}_t
    \right).
    \label{eq:reference_sampler}
\end{equation}

Let $ \mathrm{g}_t = \nabla_{\mathrm{x}_t} \mathcal{L}_t$ denote the guidance gradient. For each sample $b$, we normalize the gradient over all channels and spatial locations:
\begin{equation}
    \mathcal{N}_{\varepsilon}
    \left(
        \mathrm{g}_{t,b}
    \right)
    :=
    \frac{
        \mathrm{g}_{t,b}
    }{
        \left\|
            \mathrm{g}_{t,b}
        \right\|_2
        +
        \varepsilon
    },
    \qquad \varepsilon>0.
    \label{eq:per_sample_gradient_normalization}
\end{equation}
The reference proposal is then corrected as
\begin{equation}
    \mathrm{x}_{t-1}
    =
    \widetilde{\mathrm{x}}_{t-1}
    -
    s_t
    \mathcal{N}_{\varepsilon}
    \left(
        \mathrm{g}_t
    \right),
    \label{eq:complete_sampling_update}
\end{equation}
where $s_t\geq0$ is the guidance step size. We use a cosine schedule
that interpolates between $s_{\mathrm{early}}$ and
$s_{\mathrm{late}}$:
\begin{equation}
    s_t
    =
    s_{\mathrm{early}}
    +
    \frac{
        s_{\mathrm{late}}-s_{\mathrm{early}}
    }{2}
    \left[
        1-\cos(\pi p_t)
    \right].
    \label{eq:guidance_step_schedule}
\end{equation}
For each sample, the resulting correction satisfies
\begin{equation}
    \left\|
        \mathrm{x}_{t-1,b}
        -
        \widetilde{\mathrm{x}}_{t-1,b}
    \right\|_2
    =
    s_t
    \frac{
        \|\mathrm{g}_{t,b}\|_2
    }{
        \|\mathrm{g}_{t,b}\|_2+\varepsilon
    }
    \leq
    s_t.
    \label{eq:single_step_guidance_bound}
\end{equation}
Thus, the objective determines the correction direction, while $s_t$ provides direct control over its magnitude at each reverse step.

\paragraph{Operator-specific conjugate--gradient--Jacobian (CGJ) correction}
For truncated uniform deblurring, the implemented zero-boundary operator $\mathrm{A}$ is not diagonal in the DFT basis, whereas the spectral attenuation weight $W_{\kappa}$ is constructed from the Fourier response of the underlying blur kernel. We therefore combine measurement-space MAP refinement~\citep{li2025efficient} with Jacobian-guided propagation~\citep{song2023pseudoinverse}. The measurement-space component approximately solves
\begin{equation}
    \left[
        \bar{\alpha}_t
        \sigma_{\mathrm{y}}^2
        \mathrm{I}
        +
        \left(
            1-\bar{\alpha}_t
        \right)
        \mathrm{A}\mathrm{A}^{\top}
    \right]
    \mathrm{e}_t
    =
    \sqrt{\bar{\alpha}_t}
    \left(
        \mathrm{y}
        -
        \mathrm{A}
        \widehat{\mathrm{x}}_{0,t}
    \right)
    \label{eq:cgj_measurement_system}
\end{equation}
in the measurement space using conjugate
gradients. This avoids explicitly forming or inverting the system matrix. The solution is mapped to the clean-image space as
\begin{equation}
    \mathrm{u}_t^{\mathrm{CG}}
    =
    \mathcal{N}_{\varepsilon_{\mathrm{CG}}}
    \left(
        \mathrm{A}^{\top}\mathrm{e}_t
    \right),
    \label{eq:cgj_clean_direction}
\end{equation}
where $\varepsilon_{\mathrm{CG}}>0$ prevents division by zero.

To propagate this clean-image direction to the current diffusion state,
we follow a Jacobian-guided vector--Jacobian construction motivated by $\Pi$GDM~\citep{song2023pseudoinverse}. Using a separate computation graph, we compute the vector--Jacobian product
\begin{equation}
    \mathrm{g}_t^{\mathrm{CG}}
    =
    \nabla_{\mathrm{x}_t}
    \left[
        -
        \left\langle
            \widehat{\mathrm{x}}_{0,t}^{\mathrm{fresh}},
            \operatorname{sg}
            \left(
                \mathrm{u}_t^{\mathrm{CG}}
            \right)
        \right\rangle
    \right].
    \label{eq:cgj_vjp}
\end{equation}
This operation accounts for the dependence of the Tweedie estimate on
$\mathrm{x}_t$. The base and CGJ directions are combined as
\begin{equation}
    \overline{\mathrm{g}}_t^{\mathrm{CGJ}}
    =
    \mathcal{N}_{\varepsilon_{\mathrm{CG}}}
    \left[
        (1-\beta_{\mathrm{CG}})
        \mathcal{N}_{\varepsilon}
        \left(
            \mathrm{g}_t
        \right)
        +
        \beta_{\mathrm{CG}}
        \mathcal{N}_{\varepsilon_{\mathrm{CG}}}
        \left(
            \mathrm{g}_t^{\mathrm{CG}}
        \right)
    \right],
    \label{eq:cgj_gradient_mixture}
\end{equation}
where $\beta_{\mathrm{CG}}\in[0,1]$. When the CGJ correction is active,
$\overline{\mathrm{g}}_t^{\mathrm{CGJ}}$ replaces
$\mathcal{N}_{\varepsilon}(\mathrm{g}_t)$ in
Eq.~\eqref{eq:complete_sampling_update}. Since the combined direction is normalized again, the bound in
Eq.~\eqref{eq:single_step_guidance_bound} remains valid. The activation schedule and hyperparameters of this branch are provided
in Supplementary Sec.~\ref{app:experimental_details}.

\begin{algorithm}[t]
\caption{Frequency-decoupled guided sampling}
\label{alg:frequency_decoupled_sampling}
\small
\begin{algorithmic}[1]

\REQUIRE Observation $\mathrm{y}$, forward operator $\mathrm{A}$,
reference sampler $\mathcal{S}_{\Theta}^{(a)}$, kernel attenuation
weight $W_{\kappa}$, predefined schedules, and CGJ configuration when applicable

\STATE
$\mathrm{x}_T\sim\mathcal{N}(\mathrm{0},\mathrm{I}_d)$

\FOR{$t=T,T-1,\ldots,1$}

    \STATE Compute $p_t$, $\tau_t$, $s_t$, and
    $\{\lambda_{j,t}\}_{j\in
    \{\mathrm{lp},\mathrm{prior},\mathrm{reg}\}}$

    \STATE
    $(\widehat{\mathrm{x}}_{0,t},
      \widetilde{\mathrm{x}}_{t-1})
    \gets
    \mathcal{S}_{\Theta,t}^{(a)}(\mathrm{x}_t)$

    \STATE Construct $G_t$ and $M_t$ using
    Eqs.~\eqref{eq:lowpass_mask} and
    \eqref{eq:spectral_prior_mask}

    \STATE Evaluate $\mathcal{L}_t$ using
    Eq.~\eqref{eq:sampling_objective}

    \STATE
    $\mathrm{g}_t
    \gets
    \nabla_{\mathrm{x}_t}\mathcal{L}_t$

    \STATE
    $\overline{\mathrm{g}}_t
    \gets
    \mathcal{N}_{\varepsilon}(\mathrm{g}_t)$

    \IF{CGJ correction is active at step $t$}

        \STATE Compute
        $\overline{\mathrm{g}}_t^{\mathrm{CGJ}}$
        using
        Eqs.~\eqref{eq:cgj_measurement_system}--%
        \eqref{eq:cgj_gradient_mixture}

        \STATE
        $\overline{\mathrm{g}}_t
        \gets
        \overline{\mathrm{g}}_t^{\mathrm{CGJ}}$
    \ENDIF
    \STATE
    $\mathrm{x}_{t-1}
    \gets
    \widetilde{\mathrm{x}}_{t-1}
    -
    s_t\overline{\mathrm{g}}_t$
\ENDFOR
\ENSURE Reconstruction $\mathrm{x}_0$
\end{algorithmic}
\end{algorithm}

\paragraph{Two-branch output fusion.}
Output-space aggregation of multiple diffusion reconstructions has
previously been used to improve distortion-oriented restoration quality
in diffusion-based deblurring~\citep{whang2022deblurring,ren2023multiscale}.
Different from these sample-averaging schemes, our final reconstruction
combines two independently completed branch-specific trajectories using a task-fixed output-space weight.%

\begin{table*}[t]
\centering
\caption{
Main fidelity comparison on FFHQ and ImageNet under three deblurring
settings. Higher PSNR and SSIM are better. Bold and underlined values
indicate the best and second-best results, respectively.
}
\label{tab:quantitative_comparison}

\small
\renewcommand{\arraystretch}{1.08}
\def\mainresultstablewidth{0.8\textwidth}

\begin{tabular*}{\mainresultstablewidth}{@{\extracolsep{\fill}}lcccccc@{}}
\toprule
&
\multicolumn{2}{c}{Gaussian Blur} &
\multicolumn{2}{c}{Motion Blur} &
\multicolumn{2}{c}{Uniform Blur}
\\
\cmidrule(lr){2-3}
\cmidrule(lr){4-5}
\cmidrule(lr){6-7}

Method
& PSNR$\uparrow$ & SSIM$\uparrow$
& PSNR$\uparrow$ & SSIM$\uparrow$
& PSNR$\uparrow$ & SSIM$\uparrow$
\\
\midrule

\multicolumn{7}{c}{\textbf{FFHQ Dataset}}\\
\midrule

DPS~\citep{chung2023dps}
& 25.253 & 0.702
& 22.582 & 0.619
& 23.697 & 0.645
\\

DSG~\citep{yang2023dsg}
& 27.299 & 0.778
& 25.960 & 0.749
& 26.548 & 0.746
\\

FGPS~\citep{thaker2025fgps}
& 27.380 & 0.780
& 26.990 & 0.770
& 26.961 & 0.757
\\

DPS-MAP~\citep{li2025efficient}
& 29.274 & 0.836
& \underline{29.260} & 0.834
& 28.858 & 0.817
\\

DMPS~\citep{meng2025dmps}
& \underline{29.379} & \underline{0.839}
& 29.225 & \underline{0.835}
& \underline{28.929} & \underline{0.819}
\\

Ours
& \textbf{29.528} & \textbf{0.842}
& \textbf{29.447} & \textbf{0.839}
& \textbf{29.160} & \textbf{0.827}
\\

\midrule
\multicolumn{7}{c}{\textbf{ImageNet Dataset}}\\
\midrule

DPS~\citep{chung2023dps}
& 21.180 & 0.519
& 20.135 & 0.470
& 20.762 & 0.482
\\

DSG~\citep{yang2023dsg}
& 23.382 & 0.622
& 22.279 & 0.587
& 22.615 & 0.575
\\

FGPS~\citep{thaker2025fgps}
& 23.883 & 0.656
& 23.418 & 0.639
& 23.472 & 0.618
\\

DPS-MAP~\citep{li2025efficient}
& \underline{25.431} & \underline{0.722}
& \textbf{26.658} & \textbf{0.768}
& \textbf{25.885} & \underline{0.722}
\\

DMPS~\citep{meng2025dmps}
& 25.193 & 0.711
& 26.324 & 0.751
& 25.592 & 0.709
\\

Ours
& \textbf{25.500} & \textbf{0.725}
& \underline{26.415} & \underline{0.763}
& \underline{25.869} & \textbf{0.724}
\\

\bottomrule
\end{tabular*}
\vspace{1mm}
\begin{minipage}{\mainresultstablewidth}
\footnotesize

\end{minipage}

\end{table*}

\section{Experiments}

\subsection{Experimental Setup}
\label{sec:experimental_setup}

\paragraph{Datasets and pretrained models.}
We evaluate on Flickr-Faces-HQ (FFHQ)~\citep{karras2019style} and the ImageNet Large Scale Visual Recognition Challenge 2012 (ILSVRC2012) validation set~\citep{deng2009imagenet,russakovsky2015imagenet}, at
$256\times256$ resolution. For each dataset--task setting, we use a
fixed set of $1{,}000$ evaluation images. Within each task, all methods receive identical degraded measurements for the same images. We use a pretrained unconditional diffusion model for FFHQ and a pretrained class-conditional model with ground-truth labels for ImageNet. The dataset-specific backbone is fixed across methods, without task-specific fine-tuning. Image preprocessing is detailed in Supplementary Sec.~\ref{app:experimental_details}.
Our code will be available at \url{https://github.com/Sea-serpents/frequency-decoupled-diffusion-guidance}.

\paragraph{Tasks.}
We consider Gaussian, motion, and uniform non-blind deblurring, which
exhibit isotropic, directional, and frequency-selective attenuation,
respectively. Unless otherwise stated, measurements have additive
Gaussian noise with $\sigma_{\mathrm{y}}=0.05$. Uniform deblurring
uses a truncated forward operator. Kernel specifications, boundary
conventions, and operator construction are given in Supplementary
Sec.~\ref{app:experimental_details}.

\paragraph{Baselines and metrics.}
We compare with DPS~\citep{chung2023dps},
DSG~\citep{yang2023dsg}, FGPS~\citep{thaker2025fgps},
DPS-MAP~\citep{li2025efficient}, and DMPS~\citep{meng2025dmps}.
We report PSNR and SSIM on RGB images in $[0,1]$ as the primary
paired-fidelity metrics. LPIPS and FID results, together with their
evaluation details, are provided in Supplementary
Sec.~\ref{app:perceptual_metrics}.

\paragraph{Guidance and inference.}
We use $\rho=10^{-3}$ for the attenuation weight $W_{\kappa}$,
shared across RGB channels. Guidance coefficients follow sigmoid
schedules, and the active-band cutoff follows an exponential
schedule; their formulas and numerical settings are given in
Supplementary Sec.~\ref{app:experimental_details}.
Results labeled \emph{Ours} fuse the outputs of two completed
sampling branches: two base-guidance branches for Gaussian and motion
blur, and a base branch with a CGJ-augmented branch for uniform blur.
CGJ settings are shared across datasets; branch configurations and
fusion weights are specified in the supplementary material.

\subsection{\textbf{Results}}
Table~\ref{tab:quantitative_comparison} presents our main fidelity results
on FFHQ and ImageNet under Gaussian, motion, and uniform deblurring. Across
the six dataset--task combinations, our method achieves the best PSNR in
four cases and the best SSIM in five, while ranking second in all remaining
cases.

\begin{figure*}[t]
    \centering
    \includegraphics[width=\textwidth]
    {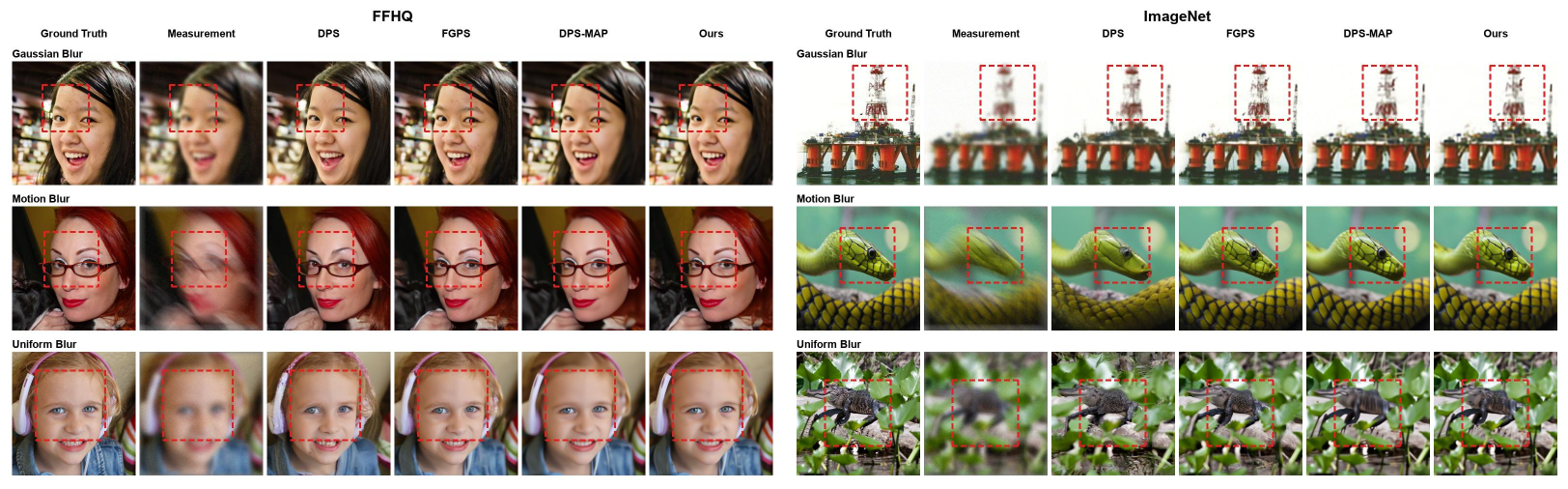}
    \caption{
    Qualitative comparison on FFHQ and ImageNet under Gaussian, motion,
    and uniform deblurring. Columns show the ground truth, measurement,
    DPS, FGPS, DPS-MAP, and our method. Red dashed boxes highlight
    matched local regions for visual comparison. Additional examples and enlarged crops are provided in the supplementary material.
    }
    \label{fig:qualitative_comparison}
\end{figure*}

\paragraph{Results on FFHQ}
Our method ranks first in both PSNR and SSIM for all three FFHQ tasks.
Relative to the strongest competing result in each setting, it improves PSNR by $0.149$, $0.187$, and $0.231$ dB for Gaussian, motion, and uniform deblurring, respectively. The improvement is therefore consistent across the three blur models and is largest under uniform deblurring, where the operator-specific CGJ branch is used. Compared with FGPS, which also uses a
progressive frequency curriculum, our method improves PSNR by more than $2$ dB and SSIM by at least $0.062$ in every FFHQ setting. These results
show that explicitly separating diffusion-time frequency activation from degradation-dependent spectral sensitivity substantially strengthens
measurement-faithful reconstruction.

\begin{table}[t]
\centering
\caption{
{Ablation of the main guidance components on the fixed 50-image FFHQ motion-deblurring subset using a single DDPM branch. }
}
\label{tab:component_ablation}

\renewcommand{\arraystretch}{1.08}

\begin{tabular}{@{}lccc@{}}
\toprule
Variant
& PSNR$\uparrow$
& SSIM$\uparrow$
& LPIPS$\downarrow$ \\
\midrule

Full
& \textbf{29.133}
& \textbf{0.8368}
& \textbf{0.1322} \\

Full $-\mathcal{P}_t$
& 29.084
& 0.8366
& 0.1342 \\

Full w/ $G_t\equiv1$
& 29.084
& 0.8365
& 0.1343 \\

Full $-\mathcal{R}_t$
& 24.446
& 0.5659
& 0.3243 \\

\bottomrule
\end{tabular}
\end{table}

\paragraph{Results on ImageNet.}
On the more diverse ImageNet dataset, our method remains within the top two for both fidelity metrics in every task. It achieves the best Gaussian deblurring result, exceeding DPS-MAP by $0.069$ dB in PSNR and $0.003$ in SSIM. %
Under uniform deblurring, the two methods are nearly tied in PSNR: our result is only $0.016$ dB below DPS-MAP, while attaining the highest SSIM of $0.724$. Thus, although ImageNet motion deblurring remains a more challenging setting, the proposed guidance gives consistently strong fidelity across datasets and operators rather than being specialized to facial images.

\paragraph{Qualitative comparison.}
Figure~\ref{fig:qualitative_comparison} provides representative visual comparisons. Our method preserves facial contours and fine structures on
FFHQ and recovers clear object details and boundaries on ImageNet, while ImageNet motion deblurring remains the most challenging setting. Additional examples and enlarged crops are provided in the supplementary material.

\subsection{Ablation Studies}
\label{sec:ablation}
We first isolate the contributions of the main guidance components and then
examine the robustness to stronger measurement
noise. The operator-specific correction and additional fine-grained ablations, {two-branch inference} analysis, and sensitivity analyses are provided in the supplementary material.

\begin{table}[t]
\centering
\caption{
Quantitative comparison on FFHQ under stronger measurement noise
$\sigma_y=0.10$.
Each result is averaged over $1{,}000$ FFHQ evaluation images, with
identical degraded measurements used across methods within each task.
}
\label{tab:strong_noise_ffhq}

\small
\setlength{\tabcolsep}{1.5pt}
\renewcommand{\arraystretch}{1.05}

\begin{tabular}{@{}lcccccc@{}}
\toprule
& \multicolumn{2}{c}{Gaussian}
& \multicolumn{2}{c}{Motion}
& \multicolumn{2}{c}{Uniform} \\
\cmidrule(lr){2-3}
\cmidrule(lr){4-5}
\cmidrule(lr){6-7}

Method
& PSNR$\uparrow$ & SSIM$\uparrow$
& PSNR$\uparrow$ & SSIM$\uparrow$
& PSNR$\uparrow$ & SSIM$\uparrow$ \\
\midrule

DPS
& 24.540 & 0.6758
& 21.784 & 0.5893
& 23.010 & 0.6204 \\

DSG
& 26.893 & 0.7596
& 25.432 & 0.7215
& 26.077 & 0.7260 \\

FGPS
& 26.736 & 0.7528
& 25.626 & 0.7216
& 25.916 & 0.7197 \\

DPS-MAP
& 28.364 & 0.8095
& 27.396 & 0.7836
& 27.414 & 0.7768 \\

DMPS
& 28.444 & 0.8120
& 27.367 & 0.7842
& 27.483 & 0.7792 \\

\textbf{Ours}
& \textbf{28.720} & \textbf{0.8198}
& \textbf{27.536} & \textbf{0.7907}
& \textbf{27.797} & \textbf{0.7918} \\

\bottomrule
\end{tabular}
\end{table}

\paragraph{Component ablation.}
Table~\ref{tab:component_ablation} compares the full method with variants that omit one component at a time.
Removing the spectral prior or disabling progressive frequency
activation reduces PSNR by about $0.05$ dB in this setting, whereas
removing the trajectory-stabilization term causes a much larger
$4.687$ dB drop together with a substantial decrease in SSIM.

\paragraph{Robustness and operator-specific correction.}
Table~\ref{tab:strong_noise_ffhq} reports the FFHQ comparison under
stronger measurement noise ($\sigma_{\mathrm y}=0.10$). Our method retains the best PSNR and SSIM on all three tasks, exceeding
the strongest competing PSNR by $0.276$, $0.140$, and $0.314$ dB for
Gaussian, motion, and uniform deblurring, respectively.

\section{Conclusion}
\label{sec:conclusion}
In this paper, we introduced frequency-decoupled posterior guidance, a training-free framework that separates the diffusion-time activation of measurement frequencies from their reliability under the degradation operator. A progressive low-to-high schedule specifies the active measurement band, while a kernel-derived attenuation map weights the remaining inactive frequencies for spectral regularization. We provided a KL-regularized path-space interpretation of the fixed endpoint energy and derived practical stepwise guidance through a Tweedie plug-in approximation. Experiments on challenging non-blind deblurring tasks demonstrate consistent improvements in reconstruction fidelity and robustness to increased measurement noise. These results highlight the importance of jointly considering the evolving frequency content of reverse diffusion and the spectral properties of the degradation operator, and motivate further extensions to blind and nonlinear inverse problems as well as more efficient single-trajectory sampling.

\clearpage

\subsubsection*{Acknowledgements}

This work was supported by the National Natural Science Foundation of China under Grant 12531011. 
The authors thank the Chern Institute of Mathematics, Nankai University, for its hospitality and support during the corresponding
author's visit, during which part of this work was carried out. We also thank Ji Li for kindly providing the DPS-MAP implementation used in our experiments and thank the authors of the open-source implementations used in this work
for making their code publicly available.

\bibliography{reference}

\clearpage
\appendix
\thispagestyle{empty}

\onecolumn
\aistatstitle{Frequency-Decoupled Diffusion Guidance for Non-Blind Image Deblurring: \\
Supplementary Materials}

This supplementary material provides theoretical proofs, experimental details, and additional analyses for the main paper. Section A presents the proofs of the path-space variational formula and the approximation error bounds. Section B details the datasets, degradation operators, method configurations, and baselines. Section C reports complementary perceptual metrics, complete component ablations, and two-branch inference analysis. Section D provides paired statistical comparisons. Section E analyzes hyperparameter sensitivity, Section F provides additional qualitative comparisons on FFHQ and ImageNet, and Section G reports detailed computational cost. Together, these results substantiate the theoretical claims and empirical effectiveness of the proposed method.

\section{Proofs}
\label{app:proof_path_regularization}
In this subsection, we prove the path-space variational formula underlying the energy tilt and
derive the resulting marginal score correction.  We then quantify the error of evaluating the frequency-decoupled energy at a conditional mean.  The
first result is the Gibbs variational principle, also known as the relative-entropy variational formula; see, e.g.,
\citep{donsker1975asymptotic,dupuis1997weak}.

\subsection{Endpoint tilting on path space}
\label{app:endpoint_tilting}

Let $(\mathcal X^{T+1},\mathcal B(\mathcal X)^{\otimes(T+1)})$ be the path
space and let $P_\Theta$ be the reference path law of
$(\mathrm X_0,\ldots,\mathrm X_T)$.  For a measurable endpoint energy
$\Phi_{\mathrm y}:\mathcal X\to[0,+\infty]$, define the extended-valued
functional
\begin{equation}
    \mathcal J_{\mathrm y}(P)
    :=
    \operatorname{KL}(P\|P_\Theta)
    +\mathbb E_P\!\left[\Phi_{\mathrm y}(\mathrm X_0)\right]
    \label{eq:app_path_objective}
\end{equation}
on probability measures $P$ on the path space.  Here,
$$
\operatorname{KL}(P\|P_\Theta)=+\infty, 
\text{ if } P\not\ll P_\Theta.
$$

\begin{proposition}[Gibbs minimizer of the path-space problem]
\label{prop:app_gibbs_path}
Assume that $\Phi_{\mathrm y}(\mathrm X_0)<+\infty$
$P_\Theta$-almost surely, and set
$Z_{\mathrm y}
    :=
    \mathbb E_{P_\Theta}
    \!\left[\exp\{-\Phi_{\mathrm y}(\mathrm X_0)\}\right].$
Then $0<Z_{\mathrm y}\leq1$, and the problem
\begin{equation}
    \inf_{P}\mathcal J_{\mathrm y}(P)
    \label{eq:app_path_problem}
\end{equation}
has the unique minimizer $P_{\mathrm y}^{\star}$ given by
\begin{equation}
    \frac{\mathrm dP_{\mathrm y}^{\star}}
         {\mathrm dP_\Theta}(\mathrm x_{0:T})
    =
    \frac{\exp\{-\Phi_{\mathrm y}(\mathrm x_0)\}}
         {Z_{\mathrm y}}.
    \label{eq:app_gibbs_path_density}
\end{equation}
Moreover,
\begin{equation}
    \inf_P\mathcal J_{\mathrm y}(P)=-\log Z_{\mathrm y}.
    \label{eq:app_gibbs_value}
\end{equation}
\end{proposition}

\begin{proof}
Since $0<\exp\{-\Phi_{\mathrm y}(\mathrm X_0)\}\leq1$
$P_\Theta$-almost surely, its expectation satisfies
$0<Z_{\mathrm y}\leq1$.  Define, for every measurable set
$B\subseteq\mathcal X^{T+1}$,
\begin{equation}
    P_{\mathrm y}^{\star}(B)
    :=
    \frac{1}{Z_{\mathrm y}}
    \int_B
    \exp\{-\Phi_{\mathrm y}(\mathrm x_0)\}
    P_\Theta(\mathrm d\mathrm x_{0:T}).
    \label{eq:app_tilt_measure_definition}
\end{equation}
The integrand is nonnegative, and the definition of $Z_{\mathrm y}$ gives
\begin{equation}
\begin{aligned}
    P_{\mathrm y}^{\star}(\mathcal X^{T+1})
    &=
    \frac{1}{Z_{\mathrm y}}
    \int_{\mathcal X^{T+1}}
    \exp\{-\Phi_{\mathrm y}(\mathrm x_0)\}
    P_\Theta(\mathrm d\mathrm x_{0:T})
    \\
    &=\frac{Z_{\mathrm y}}{Z_{\mathrm y}}=1.
\end{aligned}
\end{equation}
Thus $P_{\mathrm y}^{\star}$ is a probability measure and
\eqref{eq:app_gibbs_path_density} is its Radon--Nikodym derivative with
respect to $P_\Theta$.

We next clarify the relation between the two measures.  By construction,
$P_{\mathrm y}^{\star}\ll P_\Theta$.  Conversely, the assumption
$\Phi_{\mathrm y}(\mathrm X_0)<+\infty$ $P_\Theta$-almost surely implies
\begin{equation}
    \frac{\exp\{-\Phi_{\mathrm y}(\mathrm X_0)\}}{Z_{\mathrm y}}>0,
    \qquad P_\Theta\text{-almost surely}.
\end{equation}
Hence, if $P_{\mathrm y}^{\star}(B)=0$, then
\begin{equation}
    \int_B
    \frac{\exp\{-\Phi_{\mathrm y}(\mathrm x_0)\}}{Z_{\mathrm y}}
    P_\Theta(\mathrm d\mathrm x_{0:T})=0.
\end{equation}
Since the integrand is strictly positive $P_\Theta$-almost surely, this
identity forces $P_\Theta(B)=0$.  Therefore
$P_\Theta\ll P_{\mathrm y}^{\star}$ as well. Hence, $P_{\mathrm y}^{\star}$ and $P_\Theta$ are mutually
absolutely continuous and have the same null sets, i.e., equivalent.

By the inequality $r e^{-r}\leq e^{-1}$ for $r\geq0$, we have 
\begin{equation}
    \mathbb E_{P_{\mathrm y}^{\star}}
    [\Phi_{\mathrm y}(\mathrm X_0)]
    =\frac{1}{Z_{\mathrm y}}
      \mathbb E_{P_\Theta}
      [\Phi_{\mathrm y}(\mathrm X_0)e^{-\Phi_{\mathrm y}(\mathrm X_0)}]
    \leq \frac{1}{eZ_{\mathrm y}}<+\infty.
    \label{eq:app_tilt_energy_integrable}
\end{equation}
Furthermore,
\begin{equation}
\begin{aligned}
    \log\frac{\mathrm dP_{\mathrm y}^{\star}}{\mathrm dP_\Theta}
    &=
    \log\left(
        \frac{e^{-\Phi_{\mathrm y}(\mathrm X_0)}}{Z_{\mathrm y}}
    \right)
    \\
    &=
    -\Phi_{\mathrm y}(\mathrm X_0)-\log Z_{\mathrm y}.
\end{aligned}
\label{eq:app_log_tilt_density}
\end{equation}
The integrability established in~\eqref{eq:app_tilt_energy_integrable} gives
\begin{equation}
\begin{aligned}
    \operatorname{KL}(P_{\mathrm y}^{\star}\|P_\Theta)
    &=
    \mathbb E_{P_{\mathrm y}^{\star}}
    \left[
        \log\frac{\mathrm dP_{\mathrm y}^{\star}}{\mathrm dP_\Theta}
    \right]
    \\
    &=
    -\mathbb E_{P_{\mathrm y}^{\star}}
      [\Phi_{\mathrm y}(\mathrm X_0)]
    -\log Z_{\mathrm y}.
    \nonumber
\end{aligned}
\end{equation}
Substituting this equality into
\eqref{eq:app_path_objective} yields
\begin{equation}
\begin{aligned}
    \mathcal J_{\mathrm y}(P_{\mathrm y}^{\star})=-\log Z_{\mathrm y}.
    \nonumber
\end{aligned}
\end{equation}

Now let $P$ satisfy $\mathcal J_{\mathrm y}(P)<+\infty$.  Then
$\operatorname{KL}(P\|P_\Theta)<+\infty$ and
$\mathbb E_P[\Phi_{\mathrm y}(\mathrm X_0)]<+\infty$.  In particular,
$P\ll P_\Theta$, and mutual absolute continuity proved above gives
$P\ll P_{\mathrm y}^{\star}$.  The chain rule for Radon--Nikodym
derivatives yields, $P$-almost surely,
\begin{equation}
\begin{aligned}
    \frac{\mathrm dP}{\mathrm dP_{\mathrm y}^{\star}}
    =
    \frac{\mathrm dP}{\mathrm dP_\Theta}
    \frac{\mathrm dP_\Theta}{\mathrm dP_{\mathrm y}^{\star}}
    =
    \frac{\mathrm dP}{\mathrm dP_\Theta}
    Z_{\mathrm y}e^{\Phi_{\mathrm y}(\mathrm X_0)}.
    \label{eq:app_rn_chain_rule}
\end{aligned}
\end{equation}
Taking logarithms in~\eqref{eq:app_rn_chain_rule} gives
\begin{equation}
\begin{aligned}
    \log\frac{\mathrm dP}{\mathrm dP_{\mathrm y}^{\star}}
    &=
    \log\frac{\mathrm dP}{\mathrm dP_\Theta}
    +\Phi_{\mathrm y}(\mathrm X_0)+\log Z_{\mathrm y}.
    \label{eq:app_log_rn_chain_rule}
\end{aligned}
\end{equation}
All terms on the right-hand side are $P$-integrable under the present
finite-objective assumption.  Taking the $P$-expectation in
\eqref{eq:app_log_rn_chain_rule}, we obtain
\begin{equation}
\begin{aligned}
    \operatorname{KL}(P\|P_{\mathrm y}^{\star})
    &=
    \operatorname{KL}(P\|P_\Theta)
    +\mathbb E_P[\Phi_{\mathrm y}(\mathrm X_0)]
    +\log Z_{\mathrm y}
    \\
    &=\mathcal J_{\mathrm y}(P)+\log Z_{\mathrm y}.
\end{aligned}
    \label{eq:app_kl_decomposition}
\end{equation}
Equivalently,
\begin{equation}
    \mathcal J_{\mathrm y}(P)
    =
    \operatorname{KL}(P\|P_{\mathrm y}^{\star})-\log Z_{\mathrm y}.
\end{equation}
It then follows from the nonnegativity of relative entropy that $\mathcal J_{\mathrm y}(P)\geq-\log Z_{\mathrm y}$; this is also trivial
when $\mathcal J_{\mathrm y}(P)=+\infty$. This equality holds if and only if
$\operatorname{KL}(P\|P_{\mathrm y}^{\star})=0$, which is equivalent to
$P=P_{\mathrm y}^{\star}$ and completes the proof.
\end{proof}

\begin{corollary}[Tilted marginal and score correction]
\label{cor:app_tilted_score}
Fix $t\in\{1,\ldots,T\}$.  Suppose the reference marginal of
$\mathrm X_t$ has a strictly positive differentiable density
$p_{\Theta,t}$ and that a regular conditional distribution
$P_\Theta(\mathrm d\mathrm x_0\mid\mathrm x_t)$ exists.  Define
\begin{equation}
    Z_t(\mathrm x_t;\mathrm y)
    :=
    \int
    e^{-\Phi_{\mathrm y}(\mathrm x_0)}
    P_\Theta(\mathrm d\mathrm x_0\mid\mathrm x_t).
    \label{eq:app_local_partition}
\end{equation}
Then a version of the $t$-marginal density of $P_{\mathrm y}^{\star}$ is
\begin{equation}
    p_{\mathrm y,t}^{\star}(\mathrm x_t)
    =
    \frac{p_{\Theta,t}(\mathrm x_t)
    Z_t(\mathrm x_t;\mathrm y)}{Z_{\mathrm y}}.
    \label{eq:app_tilted_marginal}
\end{equation}
for $p_{\Theta,t}$-almost every $\mathrm{x}_t$.
At every point where $Z_t(\cdot;\mathrm y)$ is positive and differentiable,
\begin{equation}
    \nabla_{\mathrm x_t}\log p_{\mathrm y,t}^{\star}(\mathrm x_t)
    =
    \nabla_{\mathrm x_t}\log p_{\Theta,t}(\mathrm x_t)
    +\nabla_{\mathrm x_t}\log Z_t(\mathrm x_t;\mathrm y).
    \label{eq:app_exact_score_correction}
\end{equation}
\end{corollary}

\begin{proof}
Let $f$ be any bounded measurable function on $\mathcal X$.  From
\eqref{eq:app_gibbs_path_density},
\begin{equation}
\begin{aligned}
    \mathbb E_{P_{\mathrm y}^{\star}}[f(\mathrm X_t)]
    &=
    \frac{1}{Z_{\mathrm y}}
    \mathbb E_{P_\Theta}
    \left[
        f(\mathrm X_t)e^{-\Phi_{\mathrm y}(\mathrm X_0)}
    \right].
    \label{eq:app_test_function_first}
\end{aligned}
\end{equation}
Applying the tower property conditional on $\mathrm X_t$ gives
\begin{equation}
\begin{aligned}
    \mathbb E_{P_\Theta}
    \left[
        f(\mathrm X_t)e^{-\Phi_{\mathrm y}(\mathrm X_0)}
    \right]
    &=
    \mathbb E_{P_\Theta}
    \left[
        f(\mathrm X_t)
        \mathbb E_{P_\Theta}
        \left[
            e^{-\Phi_{\mathrm y}(\mathrm X_0)}
            \mid \mathrm X_t
        \right]
    \right]
    \\
    &=
    \int
        f(\mathrm x_t)Z_t(\mathrm x_t;\mathrm y)
        p_{\Theta,t}(\mathrm x_t)
        \,\mathrm d\mathrm x_t.
    \label{eq:app_disintegration_step}
\end{aligned}
\end{equation}
Combining~\eqref{eq:app_test_function_first} and
\eqref{eq:app_disintegration_step},
\begin{equation}
    \mathbb E_{P_{\mathrm y}^{\star}}[f(\mathrm X_t)]
    =
    \int f(\mathrm x_t)
    \frac{p_{\Theta,t}(\mathrm x_t)
          Z_t(\mathrm x_t;\mathrm y)}{Z_{\mathrm y}}
    \,\mathrm d\mathrm x_t.
\end{equation}
Since this identity holds for every bounded measurable $f$, the expression
inside the integral is the marginal density of $\mathrm X_t$ under
$P_{\mathrm y}^{\star}$, which proves~\eqref{eq:app_tilted_marginal}.

Taking the gradient of the logarithm of~\eqref{eq:app_tilted_marginal} completes the proof.
\end{proof}

Corollary~\ref{cor:app_tilted_score} connects the global path-space tilt to a local modification of the diffusion dynamics.  Equation
\eqref{eq:app_tilted_marginal} states that the endpoint energy reweights the
reference marginal at time $t$ by the conditional factor $Z_t(\mathrm x_t;\mathrm y)$.  Equation
\eqref{eq:app_exact_score_correction} then shows that the score of this
reweighted marginal is obtained by adding
$\nabla_{\mathrm x_t}\log Z_t(\mathrm x_t;\mathrm y)$ to the score of the reference marginal induced by the pretrained reverse process.
This additional term is the energy-guidance direction for a fixed endpoint energy. The practical method approximates this term because
the conditional expectation defining $Z_t$ is not evaluated exactly.

\subsection{Error of conditional-mean energy substitution}
\label{app:conditional_mean_substitution}
In this subsection, we analyze the approximation error introduced by evaluating the step-dependent frequency energy at the conditional mean. We further account for the discrepancy between the exact conditional mean and the learned Tweedie estimate, and relate the resulting approximation to the local log-partition function.
{All conditional expectations and covariances in this subsection are taken under the reference law $P_\Theta$.} Denote $\mathrm m_t:=
    \mathbb E_{P_\Theta}[\mathrm X_0\mid\mathrm X_t=\mathrm x_t],
    \mathbb E_t[\cdot]
    :=
    \mathbb E_{P_\Theta}[\cdot\mid\mathrm X_t=\mathrm x_t]$. We next quantify the plug-in approximation
\begin{equation}
    \mathbb E[\Phi_{t,\mathrm y}(\mathrm X_0)\mid
    \mathrm X_t=\mathrm x_t]
    \approx
    \Phi_{t,\mathrm y}(\mathrm m_t).
    \label{eq:app_mean_substitution}
\end{equation}

\begin{proposition}[Conditional-mean approximation bound]
\label{prop:app_mean_bound}
Fix $t$ and a state $\mathrm x_t$ for which the conditional second moment of
$\mathrm X_0$ is finite.  Let $N$ denote the number of scalar Fourier coefficients over which the averaged norms in the main text are computed, and let $\mathcal F$ be the DFT used in the main text.  Assume
$\lambda_{\mathrm{lp},t},\lambda_{\mathrm{prior},t}\geq0$ and set
\begin{equation}
\begin{aligned}
    M_t&:=(1-G_t)W_\kappa,
    &
    \mathcal K_t&:=\operatorname{Diag}(G_t)\mathcal F\mathrm A,
    \\
    \mathrm b_t&:=\operatorname{Diag}(G_t)\mathcal F\mathrm y,
    &
    \mathcal D_t&:=\operatorname{Diag}(M_t)\mathcal F.
\end{aligned}
\label{eq:app_frequency_operators}
\end{equation}
Then
\begin{equation}
\begin{aligned}
0
&\leq
\mathbb E[\Phi_{t,\mathrm y}(\mathrm X_0)\mid
\mathrm X_t=\mathrm x_t]
-\Phi_{t,\mathrm y}(\mathrm m_t)
\\
&\leq
\frac{\lambda_{\mathrm{lp},t}}{N}
\operatorname{tr}(\mathcal K_t\mathrm C_t\mathcal K_t^*)
+
\frac{\lambda_{\mathrm{prior},t}}{\sqrt N}
\sqrt{\operatorname{tr}(\mathcal D_t\mathrm C_t\mathcal D_t^*)},
\end{aligned}
\label{eq:app_frequency_energy_error_bound}
\end{equation}
where $\mathrm C_t
    :=\operatorname{Cov}_{P_\Theta}(\mathrm X_0\mid\mathrm X_t=\mathrm x_t).$
In particular, the mean substitution is exact if
$\mathcal K_t\mathrm C_t\mathcal K_t^*=0$ and
$\mathcal D_t\mathrm C_t\mathcal D_t^*=0$.
\end{proposition}

\begin{proof}
Recall the frequency-decoupled energy from the main text
\begin{equation}
\begin{aligned}
    \Phi_{t,\mathrm y}(\mathrm x)
    =
    \lambda_{\mathrm{lp},t}
    \left\|
        G_t\odot
        \mathcal F(\mathrm A\mathrm x-\mathrm y)
    \right\|_{2,\mathrm{avg}}^2
    +
    \lambda_{\mathrm{prior},t}
    \left\|
        M_t\odot\mathcal F(\mathrm x)
    \right\|_{1,\mathrm{avg}}.
    \label{eq:app_recalled_frequency_energy}
\end{aligned}
\end{equation}
By the notations in (\ref{eq:app_frequency_operators}),
\begin{equation}
\begin{aligned}
    G_t\odot\mathcal F(\mathrm A\mathrm x-\mathrm y)
    &=
    \operatorname{Diag}(G_t)
    \bigl(\mathcal F\mathrm A\mathrm x-\mathcal F\mathrm y\bigr)
    \\
    &=
    \operatorname{Diag}(G_t)\mathcal F\mathrm A\mathrm x
    -\operatorname{Diag}(G_t)\mathcal F\mathrm y
    \\
    &=\mathcal K_t\mathrm x-\mathrm b_t,
    \\
    M_t\odot\mathcal F(\mathrm x)
    &=
    \operatorname{Diag}(M_t)\mathcal F\mathrm x
    =\mathcal D_t\mathrm x.
\end{aligned}
\nonumber
\end{equation}
Substituting these identities into
\eqref{eq:app_recalled_frequency_energy} gives
\begin{equation}
    \Phi_{t,\mathrm y}(\mathrm x)
    =
    \frac{\lambda_{\mathrm{lp},t}}{N}
    \|\mathcal K_t\mathrm x-\mathrm b_t\|_2^2
    +
    \frac{\lambda_{\mathrm{prior},t}}{N}
    \|\mathcal D_t\mathrm x\|_1.
    \label{eq:app_operator_energy}
\end{equation}

Let $ \mathrm V:=\mathrm X_0-\mathrm m_t.$
Then we have
\begin{equation}
    \mathbb E_t[\mathrm V]
    =
    \mathbb E_t[\mathrm X_0]-\mathrm m_t
    =0,
    \qquad
    \mathbb E_t[\mathrm V\mathrm V^*]=\mathrm C_t.
    \label{eq:app_centered_moments}
\end{equation}
The finite conditional second moment and finite dimensionality imply that
all terms below are integrable.

We first prove the lower bound.  The mapping
\begin{equation}
    \mathrm x\longmapsto
    \|\mathcal K_t\mathrm x-\mathrm b_t\|_2^2
\end{equation}
is convex because it is the composition of a convex squared norm with an
affine mapping.  Similarly,
\begin{equation}
    \mathrm x\longmapsto\|\mathcal D_t\mathrm x\|_1
\end{equation}
is convex because it is the composition of a norm with a linear mapping.
Since both weights in~\eqref{eq:app_operator_energy} are nonnegative,
$\Phi_{t,\mathrm y}$ is convex.  Conditional Jensen's inequality now gives
\begin{equation}
\begin{aligned}
    \Phi_{t,\mathrm y}(\mathrm m_t)
    =
    \Phi_{t,\mathrm y}(\mathbb E_t[\mathrm X_0])
    \leq
    \mathbb E_t[\Phi_{t,\mathrm y}(\mathrm X_0)],
\end{aligned}
\label{eq:app_conditional_jensen}
\end{equation}
which is the first inequality in
\eqref{eq:app_frequency_energy_error_bound}.

We next evaluate the gap of the quadratic measurement term.  Set
\begin{equation}
    \mathrm a_t:=\mathcal K_t\mathrm m_t-\mathrm b_t.
\end{equation}
Since $\mathrm X_0=\mathrm m_t+\mathrm V$,
\begin{equation}
    \mathcal K_t\mathrm X_0-\mathrm b_t
    =
    \mathrm a_t+\mathcal K_t\mathrm V.
\end{equation}
Expanding the squared norm in the complex Euclidean inner product yields
\begin{equation}
\begin{aligned}
    \|\mathcal K_t\mathrm X_0-\mathrm b_t\|_2^2
    ={}&
    \|\mathrm a_t\|_2^2
    +2\operatorname{Re}
       \langle\mathrm a_t,\mathcal K_t\mathrm V\rangle
    +\|\mathcal K_t\mathrm V\|_2^2.
    \label{eq:app_quadratic_expansion}
\end{aligned}
\end{equation}
The cross term has zero conditional expectation because
\begin{equation}
\begin{aligned}
    \mathbb E_t
    [\langle\mathrm a_t,\mathcal K_t\mathrm V\rangle]
    &=
    \left\langle
        \mathrm a_t,
        \mathcal K_t\mathbb E_t[\mathrm V]
    \right\rangle
    =0,
\end{aligned}
\end{equation}
where we used linearity of $\mathcal K_t$ and
\eqref{eq:app_centered_moments}.  Taking conditional expectation in
\eqref{eq:app_quadratic_expansion} therefore gives
\begin{equation}
\begin{aligned}
\mathbb E_t\|\mathcal K_t\mathrm X_0-\mathrm b_t\|_2^2
-\|\mathcal K_t\mathrm m_t-\mathrm b_t\|_2^2
=
\mathbb E_t\|\mathcal K_t\mathrm V\|_2^2.
\end{aligned}
\label{eq:app_quadratic_gap_first}
\end{equation}
To express the last term through the covariance, use
\begin{equation}
\begin{aligned}
    \|\mathcal K_t\mathrm V\|_2^2
    =
    (\mathcal K_t\mathrm V)^*(\mathcal K_t\mathrm V)
    =
    \operatorname{tr}
    \left(
        \mathcal K_t\mathrm V\mathrm V^*\mathcal K_t^*
    \right).
\end{aligned}
\end{equation}
Taking conditional expectations, using linearity of the trace and
\eqref{eq:app_centered_moments}, we obtain
\begin{equation}
\begin{aligned}
    \mathbb E_t\|\mathcal K_t\mathrm V\|_2^2
    &=
    \operatorname{tr}
    \left(
        \mathcal K_t
        \mathbb E_t[\mathrm V\mathrm V^*]
        \mathcal K_t^*
    \right)
    \\
    &=
    \operatorname{tr}
    (\mathcal K_t\mathrm C_t\mathcal K_t^*).
\end{aligned}
\label{eq:app_quadratic_exact_gap}
\end{equation}

We finally estimate the spectral-prior gap.  Since
$\mathcal D_t\mathrm X_0
=\mathcal D_t\mathrm m_t+\mathcal D_t\mathrm V$, 
\begin{equation}
    \|\mathcal D_t\mathrm X_0\|_1
    \leq
    \|\mathcal D_t\mathrm m_t\|_1
    +\|\mathcal D_t\mathrm V\|_1.
\end{equation}
After subtracting $\|\mathcal D_t\mathrm m_t\|_1$ and taking conditional
expectations, this gives
\begin{equation}
\begin{aligned}
\mathbb E_t\|\mathcal D_t\mathrm X_0\|_1
-\|\mathcal D_t\mathrm m_t\|_1
\leq
\mathbb E_t\|\mathcal D_t\mathrm V\|_1.
\end{aligned}
\label{eq:app_l1_triangle_step}
\end{equation}
For any $\mathrm z\in\mathbb C^N$, the Cauchy--Schwarz inequality gives
\begin{equation}
    \|\mathrm z\|_1
    =\sum_{j=1}^N|z_j|
    \leq
    \left(\sum_{j=1}^N1^2\right)^{1/2}
    \left(\sum_{j=1}^N|z_j|^2\right)^{1/2}
    =\sqrt N\|\mathrm z\|_2.
\end{equation}
Applying this inequality to $\mathcal D_t\mathrm V$ and then applying
Cauchy--Schwarz to the scalar random variable
$\|\mathcal D_t\mathrm V\|_2$ yields
\begin{equation}
\begin{aligned}
    \mathbb E_t\|\mathcal D_t\mathrm V\|_1
    \leq
    \sqrt N\,\mathbb E_t\|\mathcal D_t\mathrm V\|_2
    \leq
    \sqrt N
    \left(
        \mathbb E_t\|\mathcal D_t\mathrm V\|_2^2
    \right)^{1/2}.
\end{aligned}
\label{eq:app_l1_to_l2_step}
\end{equation}
As in the quadratic calculation,
\begin{equation}
\begin{aligned}
    \mathbb E_t\|\mathcal D_t\mathrm V\|_2^2
    =
    \mathbb E_t
    \operatorname{tr}
    (\mathcal D_t\mathrm V\mathrm V^*\mathcal D_t^*)
    =
    \operatorname{tr}
    (\mathcal D_t\mathrm C_t\mathcal D_t^*).
\end{aligned}
\label{eq:app_prior_trace_identity}
\end{equation}
Combining~\eqref{eq:app_l1_triangle_step}--
\eqref{eq:app_prior_trace_identity}, we get
\begin{equation}
\begin{aligned}
\mathbb E_t\|\mathcal D_t\mathrm X_0\|_1
-\|\mathcal D_t\mathrm m_t\|_1
\leq
\sqrt N\,
\sqrt{\operatorname{tr}(\mathcal D_t\mathrm C_t\mathcal D_t^*)}.
\end{aligned}
\label{eq:app_l1_gap}
\end{equation}
Multiplying~\eqref{eq:app_quadratic_exact_gap} by
$\lambda_{\mathrm{lp},t}/N$, multiplying~\eqref{eq:app_l1_gap} by
$\lambda_{\mathrm{prior},t}/N$, and adding the resulting inequalities gives
\begin{equation}
\begin{aligned}
\mathbb E_t[\Phi_{t,\mathrm y}(\mathrm X_0)]
-\Phi_{t,\mathrm y}(\mathrm m_t)
\leq
\frac{\lambda_{\mathrm{lp},t}}{N}
\operatorname{tr}(\mathcal K_t\mathrm C_t\mathcal K_t^*)
+
\frac{\lambda_{\mathrm{prior},t}}{\sqrt N}
\sqrt{\operatorname{tr}(\mathcal D_t\mathrm C_t\mathcal D_t^*)},
\end{aligned}
\end{equation}
which is the required upper bound.
\end{proof}

\begin{corollary}[Additional error from an approximate Tweedie estimate]
\label{cor:app_denoiser_error}
Under the assumptions of Proposition~\ref{prop:app_mean_bound}, let
$\widehat{\mathrm x}_{0,t}=\mathrm m_t+\boldsymbol{\eta}_t$.  Then
\begin{equation}
\begin{aligned}
&\left|
\mathbb E_t[\Phi_{t,\mathrm y}(\mathrm X_0)]
-\Phi_{t,\mathrm y}(\widehat{\mathrm x}_{0,t})
\right|
\\
\leq &
\Delta_t^{\mathrm{cm}}
+\frac{\lambda_{\mathrm{lp},t}}{N}
\left(
2\|\mathcal K_t\mathrm m_t-\mathrm b_t\|_2
  \|\mathcal K_t\boldsymbol{\eta}_t\|_2
+\|\mathcal K_t\boldsymbol{\eta}_t\|_2^2
\right) 
+\frac{\lambda_{\mathrm{prior},t}}{N}
\|\mathcal D_t\boldsymbol{\eta}_t\|_1,
\end{aligned}
\label{eq:app_approx_tweedie_bound}
\end{equation}
where $\Delta_t^{\mathrm{cm}}$ denotes the right-hand side of
\eqref{eq:app_frequency_energy_error_bound}.
\end{corollary}

\begin{proof}
Adding and subtracting $\Phi_{t,\mathrm y}(\mathrm m_t)$ and applying the
triangle inequality gives
\begin{equation}
\begin{aligned}
&\left|
\mathbb E_t[\Phi_{t,\mathrm y}(\mathrm X_0)]
-\Phi_{t,\mathrm y}(\widehat{\mathrm x}_{0,t})
\right|
\\
\leq &
\left|
\mathbb E_t[\Phi_{t,\mathrm y}(\mathrm X_0)]
-\Phi_{t,\mathrm y}(\mathrm m_t)
\right|
+
\left|
\Phi_{t,\mathrm y}(\mathrm m_t)
-\Phi_{t,\mathrm y}(\mathrm m_t+\boldsymbol{\eta}_t)
\right|.
\label{eq:app_denoiser_triangle}
\end{aligned}
\end{equation}
The first term is nonnegative and bounded by
$\Delta_t^{\mathrm{cm}}$ according to Proposition
\ref{prop:app_mean_bound}.

For the quadratic part of the second term, set
$\mathrm a_t:=\mathcal K_t\mathrm m_t-\mathrm b_t$.  Then
\begin{equation}
\begin{aligned}
&\left|
\|\mathcal K_t(\mathrm m_t+\boldsymbol{\eta}_t)-\mathrm b_t\|_2^2
-\|\mathcal K_t\mathrm m_t-\mathrm b_t\|_2^2
\right|
\\
&\quad=
\left|
2\operatorname{Re}
\langle\mathrm a_t,\mathcal K_t\boldsymbol{\eta}_t\rangle
+\|\mathcal K_t\boldsymbol{\eta}_t\|_2^2
\right|
\\
&\quad\leq
2\|\mathrm a_t\|_2\|\mathcal K_t\boldsymbol{\eta}_t\|_2
+\|\mathcal K_t\boldsymbol{\eta}_t\|_2^2,
\end{aligned}
\label{eq:app_denoiser_quadratic_error}
\end{equation}
where the last step uses
$|\operatorname{Re}\langle\mathrm u,\mathrm v\rangle|
\leq\|\mathrm u\|_2\|\mathrm v\|_2$.
For the spectral-prior term, the reverse triangle inequality gives
\begin{equation}
\begin{aligned}
&\left|
\|\mathcal D_t(\mathrm m_t+\boldsymbol{\eta}_t)\|_1
-\|\mathcal D_t\mathrm m_t\|_1
\right|
\leq
\|\mathcal D_t\boldsymbol{\eta}_t\|_1.
\end{aligned}
\label{eq:app_denoiser_prior_error}
\end{equation}
Substituting~\eqref{eq:app_denoiser_quadratic_error} and
\eqref{eq:app_denoiser_prior_error} into
\eqref{eq:app_denoiser_triangle}, together with the normalized weights in
\eqref{eq:app_operator_energy}, proves~\eqref{eq:app_approx_tweedie_bound}.
\end{proof}

For the variance-preserving Gaussian diffusion model considered here,
$\mathrm X_t=\sqrt{\bar\alpha_t}\mathrm X_0+
\sqrt{1-\bar\alpha_t}\,\boldsymbol\epsilon$, Tweedie's formula gives:
\begin{equation}
    \mathrm m_t
    =\mathbb E[\mathrm X_0\mid\mathrm X_t=\mathrm x_t]
    =
    \frac{
    \mathrm x_t+(1-\bar\alpha_t)
    \nabla_{\mathrm x_t}\log p_t(\mathrm x_t)}
    {\sqrt{\bar\alpha_t}}.
    \label{eq:app_tweedie_identity}
\end{equation}
Replacing the exact score by a learned score produces
$\widehat{\mathrm x}_{0,t}$; Corollary~\ref{cor:app_denoiser_error} quantifies the additional
energy error associated with the discrepancy
$\boldsymbol{\eta}_t
=
\widehat{\mathrm x}_{0,t}-\mathrm m_t$.

We conclude by distinguishing the energy mean from the local log-partition
quantity that enters the exact score correction.  Define the conditional
free energy
\begin{equation}
    \mathcal A_t^{\mathrm{loc}}(\mathrm x_t;\mathrm y)
    :=
    -\log Z_t^{\mathrm{loc}}(\mathrm x_t;\mathrm y)
    =
    -\log\mathbb E_t
    \left[
        e^{-\Phi_{t,\mathrm y}(\mathrm X_0)}
    \right].
    \label{eq:app_conditional_free_energy}
\end{equation}
Since $r\mapsto e^{-r}$ is convex, Jensen's inequality gives
\begin{equation}
    e^{-\mathbb E_t[\Phi_{t,\mathrm y}(\mathrm X_0)]}
    \leq
    \mathbb E_t[e^{-\Phi_{t,\mathrm y}(\mathrm X_0)}].
\end{equation}
Because $-\log$ is decreasing, applying it to both sides yields
\begin{equation}
    \mathcal A_t^{\mathrm{loc}}(\mathrm x_t;\mathrm y)
    \leq
    \mathbb E_t[\Phi_{t,\mathrm y}(\mathrm X_0)]
    \leq
    \Phi_{t,\mathrm y}(\mathrm m_t)+\Delta_t^{\mathrm{cm}},
    \label{eq:app_log_partition_jensen}
\end{equation}
where the second inequality follows from Proposition
\ref{prop:app_mean_bound}.

A direct concentration condition makes the plug-in interpretation precise.
For example, suppose that, conditionally on
$\mathrm X_t=\mathrm x_t$,
\begin{equation}
    \left|
        \Phi_{t,\mathrm y}(\mathrm X_0)
        -\Phi_{t,\mathrm y}(\mathrm m_t)
    \right|
    \leq\delta_t
    \qquad\text{almost surely}.
    \label{eq:app_energy_concentration}
\end{equation}
Exponentiating the two inequalities implied by
\eqref{eq:app_energy_concentration} and then taking conditional expectations
gives
\begin{equation}
\begin{aligned}
    e^{-\Phi_{t,\mathrm y}(\mathrm m_t)-\delta_t}
    &\leq
    Z_t^{\mathrm{loc}}(\mathrm x_t;\mathrm y)
    \leq
    e^{-\Phi_{t,\mathrm y}(\mathrm m_t)+\delta_t}.
\end{aligned}
\end{equation}
Applying $-\log$ therefore yields the two-sided estimate
\begin{equation}
    \left|
        \mathcal A_t^{\mathrm{loc}}(\mathrm x_t;\mathrm y)
        -\Phi_{t,\mathrm y}(\mathrm m_t)
    \right|
    \leq\delta_t.
    \label{eq:app_free_energy_concentration_bound}
\end{equation}
Thus replacing the conditional log-partition function by the energy at the
conditional mean is justified when the conditional energy is concentrated.
If $\widehat{\mathrm x}_{0,t}=\mathrm m_t+\boldsymbol{\eta}_t$, combining
\eqref{eq:app_free_energy_concentration_bound} with
\eqref{eq:app_denoiser_quadratic_error} and
\eqref{eq:app_denoiser_prior_error} gives
\begin{equation}
\begin{aligned}
&\left|
\mathcal A_t^{\mathrm{loc}}(\mathrm x_t;\mathrm y)
-\Phi_{t,\mathrm y}(\widehat{\mathrm x}_{0,t})
\right|
\\
\leq &
\delta_t
+\frac{\lambda_{\mathrm{lp},t}}{N}
\left(
2\|\mathcal K_t\mathrm m_t-\mathrm b_t\|_2
  \|\mathcal K_t\boldsymbol{\eta}_t\|_2
+\|\mathcal K_t\boldsymbol{\eta}_t\|_2^2
\right) \\
&
+\frac{\lambda_{\mathrm{prior},t}}{N}
\|\mathcal D_t\boldsymbol{\eta}_t\|_1.
\end{aligned}
\nonumber
\end{equation}
This bound separates the conditional-concentration error $\delta_t$ from
the error of the learned Tweedie estimate.  Accordingly,
\begin{equation}
    Z_t^{\mathrm{loc}}(\mathrm x_t;\mathrm y)
\approx
e^{-\Phi_{t,\mathrm y}(\widehat{\mathrm x}_{0,t})}.
\end{equation}
is a concentration-based plug-in approximation, rather than an exact
consequence of Jensen's inequality.

\section{Experimental Details}
\label{app:experimental_details}

\subsection{Datasets and Evaluation Protocol}

We evaluate on Flickr-Faces-HQ (FFHQ)~\citep{karras2019style} and the
ILSVRC2012 validation set~\citep{deng2009imagenet,russakovsky2015imagenet}
at $256\times256$ resolution. For each dataset--task setting, we use a
fixed set of $1{,}000$ evaluation images. Within each task, all methods
receive the same degraded measurement of each image. Images are represented
in RGB. FFHQ images are resized to $256\times256$ when necessary; ImageNet
images are bicubically resized to $256\times256$ without cropping.

For FFHQ, we use the pretrained unconditional \texttt{ffhq\_10m.pt}
checkpoint distributed with the official DPS release~\citep{chung2023dps}.
For ImageNet, we use the pretrained $256\times256$ class-conditional
diffusion checkpoint released with Guided Diffusion~\citep{dhariwal2021diffusion},
conditioning on the ground-truth class labels. The backbone is fixed
across methods within each dataset, without task-specific fine-tuning.

\paragraph{External assets and licenses.}
The FFHQ dataset resources are distributed under the CC BY-NC-SA 4.0
license, while individual images retain their original Flickr licenses.
ImageNet is used under its documented access terms for non-commercial
research and educational purposes.

\paragraph{Sensitive content.}
FFHQ contains human face images and may therefore involve privacy-sensitive
visual content. We use FFHQ solely as an image-restoration benchmark and do
not perform identity recognition, demographic inference, or analysis of
personal attributes. We do not collect or curate any new human-subject data.

We compute PSNR and SSIM for each reconstructed image against its
reference in RGB $[0,1]$ and average the scores over the evaluation
set. LPIPS is computed per image using the AlexNet backbone after
mapping the reconstruction and reference to $[-1,1]$, and is averaged
over the same set. FID is computed between the complete reconstructed
and reference image sets; its implementation and the perceptual results
are reported in Supplementary Sec.~\ref{app:perceptual_metrics}.

\subsection{Degradation Operators}

We study Gaussian, motion, and uniform non-blind deblurring. Unless
otherwise stated, the measurement follows Eq.~\eqref{eq:model} with
additive Gaussian noise of standard deviation
$\sigma_{\mathrm{y}}=0.05$. Each task uses the same realized
measurements for all compared methods.

\paragraph{Gaussian deblurring.}
The blur kernel has size $61\times61$ and standard deviation $3.0$.
The FFHQ forward operator uses zero padding, whereas the ImageNet
operator uses reflection padding.

\paragraph{Motion deblurring.}
We use a $61\times61$ random-trajectory kernel generated once with
intensity $0.5$ and kernel seed $0$. It is fixed across the evaluated
images and applied using zero-boundary convolution.

\paragraph{Uniform deblurring.}
The underlying kernel is a $9\times9$ uniform averaging kernel. The
implemented forward operator uses a separable SVD construction with
zero-boundary convolution. Specifically, let $\mathrm{B}$ be the
one-dimensional zero-boundary convolution matrix for the nine-point
averaging filter, and write
$\mathrm{B}=\mathrm{U}\operatorname{diag}(s_i)\mathrm{V}^{\top}$.
Set
$\mathrm{B}_{0.03}=\mathrm{U}\operatorname{diag}
\bigl(s_i\mathbf{1}_{\{s_i\geq0.03\}}\bigr)\mathrm{V}^{\top}$
and apply the separable operator
$\mathrm{A}(\mathrm{X})=\mathrm{B}_{0.03}\mathrm{X}
\mathrm{B}_{0.03}^{\top}$ independently to each color channel.
Thus, singular values strictly below $0.03$ are discarded in each
one-dimensional factor before forming the two-dimensional operator.
This truncated operator is used both to generate measurements and
to evaluate measurement residuals. The attenuation weight
$W_{\kappa}$ in Eq.~\eqref{eq:kernel_attenuation_weight} is instead
computed from the Fourier response of the underlying $9\times9$
kernel. Thus, it serves as a kernel-based spectral proxy and is not
the singular-value spectrum of the implemented forward operator.

\subsection{Method Configurations}

Both reconstruction branches use $T=1{,}000$ reverse steps; the DDIM
branch is deterministic ($\eta=0$). We set $\rho=10^{-3}$ in
$W_{\kappa}$, using the same spatial-frequency weight for all three
color channels. For $t=T,T-1,\ldots,1$, define the reverse progress
$p_t=(T-t)/(T-1)$. The reported temporal coefficients are
\begin{equation}
    \lambda_{j,t}
    =
    \lambda_j\operatorname{sigm}
    \bigl(k_j(p_t-p_j)\bigr),
    \qquad
    j\in\{\mathrm{lp},\mathrm{prior},\mathrm{reg}\},
    \label{eq:app_temporal_term_gates}
\end{equation}
where $\operatorname{sigm}(u)=(1+e^{-u})^{-1}$. We use the same
activation centers and slopes for all tasks:
\begin{equation}
\begin{aligned}
    (p_{\mathrm{lp}},k_{\mathrm{lp}})&=(-0.2,10),\\
    (p_{\mathrm{prior}},k_{\mathrm{prior}})&=(0.2,10),\\
    (p_{\mathrm{reg}},k_{\mathrm{reg}})&=(0.55,12).
\end{aligned}
\end{equation}
The active-band cutoff and normalized-guidance step size follow
\begin{align}
    \tau_t
    &=
    \tau_{\mathrm{end}}
    -(\tau_{\mathrm{end}}-\tau_{\mathrm{start}})e^{-5p_t},
    \label{eq:app_exponential_cutoff_schedule}\\
    s_t
    &=
    s_{\mathrm{early}}
    +\frac{s_{\mathrm{late}}-s_{\mathrm{early}}}{2}
    \bigl[1-\cos(\pi p_t)\bigr].
    \label{eq:app_guidance_step_schedule}
\end{align}
The cutoff therefore starts at $\tau_{\mathrm{start}}$ and approaches,
but does not generally reach, $\tau_{\mathrm{end}}$ at the final step.
Cutoffs are radii in frequency-bin units on the centered
$256\times256$ DFT grid.
Gradient normalization is applied independently to each sample over
channels and spatial locations with safeguard $\varepsilon=10^{-8}$.
All reported configurations use the clean-image spectral $\ell_1$
prior; the motion-specific late Fourier auxiliary is disabled.
Table~\ref{tab:method_configurations} gives the task-specific
parameters and output-fusion weights.

\begin{table*}[t]
\centering
\caption{Task-specific configurations. A combined DDPM/DDIM row means
that both branches use the listed guidance parameters. For uniform
deblurring, CGJ is enabled only in the CGJ-DDPM branch. The weight
$w$ applies to the DDPM or CGJ-DDPM output; the DDIM output receives
$1-w$. It is repeated on separate branch rows for clarity.}
\label{tab:method_configurations}
\small
\setlength{\tabcolsep}{4.0pt}
\renewcommand{\arraystretch}{1.08}
\begin{tabular}{llccccccc}
\toprule
Dataset & Task and branch
& $\tau_{\mathrm{start}}$
& $\tau_{\mathrm{end}}$
& $\lambda_{\mathrm{lp}}$
& $\lambda_{\mathrm{prior}}$
& $\lambda_{\mathrm{reg}}$
& $(s_{\mathrm{early}},s_{\mathrm{late}})$
& $w$ \\
\midrule
FFHQ & Gaussian DDPM/DDIM
& 10 & 75 & 1.30 & 0.0030 & 0.0050 & $(2.5,0.60)$ & 0.50 \\
FFHQ & Motion DDPM/DDIM
& 8 & 125 & 2.20 & 0.0050 & 0.0050 & $(2.5,0.60)$ & 0.50 \\
FFHQ & Uniform CGJ-DDPM
& 8 & 115 & 1.20 & 0.0005 & 0.0030 & $(2.5,0.60)$ & 0.50 \\
FFHQ & Uniform DDIM
& 8 & 115 & 1.20 & 0.0005 & 0.0030 & $(2.5,0.60)$ & 0.50 \\
\midrule
ImageNet & Gaussian DDPM/DDIM
& 10 & 75 & 1.00 & 0.0010 & 0.0030 & $(2.0,0.50)$ & 0.70 \\
ImageNet & Motion DDPM
& 8 & 105 & 2.20 & 0.00125 & 0.0050 & $(2.5,0.60)$ & 0.50 \\
ImageNet & Motion DDIM
& 8 & 115 & 2.20 & 0.000375 & 0.0030 & $(2.5,0.60)$ & 0.50 \\
ImageNet & Uniform CGJ-DDPM
& 8 & 105 & 1.50 & 0.0003 & 0.0020 & $(2.7,0.65)$ & 0.55 \\
ImageNet & Uniform DDIM
& 8 & 100 & 1.35 & 0.0003 & 0.0025 & $(2.6,0.60)$ & 0.55 \\
\bottomrule
\end{tabular}
\end{table*}

The two branches are run to completion independently and their final saved 8-bit RGB images are fused using the fixed task-specific output-space weights reported in Table~\ref{tab:method_configurations}.
Consequently, the reported \emph{Ours} result uses two standard
DDPM/DDIM branches for Gaussian and motion deblurring, and a
CGJ-DDPM/DDIM pair for uniform deblurring. The output weight $w$ is
distinct from the CGJ gradient-mixture weight $\beta_{\mathrm{CG}}$.

\paragraph{Anisotropic TV.}
For a multi-channel image $\mathrm{u}$, the trajectory regularizer
uses
\begin{equation}
    \operatorname{TV}_{\mathrm{ani}}(\mathrm{u})
    =
    \operatorname{mean}\bigl(|D_{\mathrm{h}}\mathrm{u}|\bigr)
    +
    \operatorname{mean}\bigl(|D_{\mathrm{v}}\mathrm{u}|\bigr),
    \label{eq:anisotropic_tv_impl}
\end{equation}
where $D_{\mathrm{h}}$ and $D_{\mathrm{v}}$ are horizontal and vertical
first-order finite differences, and each mean is taken over the
corresponding spatial differences and color channels. In
Eq.~\eqref{eq:trajectory_regularization}, this is applied to
$\mathrm{x}_t-\operatorname{sg}(\widehat{\mathrm{x}}_{0,t})$;
the clean-image anchor is detached for this term.

\paragraph{CGJ correction.}
CGJ is applied only to the DDPM branch of uniform deblurring on both
datasets. We set $\beta_{\mathrm{CG}}=0.3$ and
$\varepsilon_{\mathrm{CG}}=10^{-12}$. The correction is activated at
$p_t\geq0.25$ and applied once every five reverse steps, with five
conjugate-gradient iterations per activation. These settings are
shared between FFHQ and ImageNet.

\subsection{Baseline Configurations}

We use public implementations of the five baselines with the same
dataset-specific diffusion backbone, evaluation images, and degraded
measurements as above. Their reverse-step counts and guidance settings
are given separately because they differ across methods and tasks.

\paragraph{DPS.}
We run DPS~\citep{chung2023dps} for $1{,}000$ DDPM reverse steps with
guidance scale $0.3$. We use the released Gaussian and motion task
configurations for FFHQ and ImageNet. For uniform deblurring, we
substitute the truncated operator above while keeping the corresponding
backbone and sampling protocol.

\paragraph{DSG.}
We run DSG~\citep{yang2023dsg} for $1{,}000$ reverse steps. On ImageNet,
we use the matched class-conditional backbone, retaining the guidance
rule and frozen hyperparameters. For uniform deblurring, we replace
the forward operator by the truncated uniform operator but use the
corresponding frozen Gaussian-deblurring hyperparameters without
uniform-specific tuning.

\paragraph{FGPS.}
We run FGPS~\citep{thaker2025fgps} for $1{,}000$ reverse steps. On FFHQ,
we use one frozen configuration across the three tasks: scale $0.4$,
frequency bounds $(10,75)$, and a guidance scale decreasing linearly
from $3.0$ to $0.6$. For ImageNet uniform deblurring, we apply the
frozen ImageNet Gaussian configuration with the uniform operator,
without additional tuning.

\paragraph{DPS-MAP.}
We run the released DPS-MAP implementation~\citep{li2025efficient}
with guidance scale $1.0$ and $20$ conjugate-gradient iterations for
its measurement-space correction. The FFHQ runs use $1{,}000$
respaced DDIM steps and the ImageNet runs use $100$ respaced DDIM
steps. The measurement operators and noise levels match those used
for the other methods.

\paragraph{DMPS.}
We run the public DMPS implementation~\citep{meng2025dmps} with
guidance scale $1.0$. All FFHQ tasks and ImageNet Gaussian deblurring use $100$ respaced reverse steps; ImageNet motion and uniform deblurring use
$1{,}000$ reverse steps. We use the corresponding measurement
operators described above. On ImageNet, the class-conditional model
receives the ground-truth class labels.

\section{Additional Experimental Results}
\subsection{Complementary Perceptual Metrics}
\label{app:perceptual_metrics}
Because our primary objective in non-blind image
inverse problems is faithful recovery of the specific image associated with each measurement, we first report PSNR and SSIM in the main text. The complementary LPIPS and FID results are reported in  Table~\ref{tab:perceptual_metrics}. 
The results reveal a clear separation between paired fidelity and perceptual
quality. Although DPS-MAP and DMPS are among the strongest baselines in PSNR
and SSIM, their LPIPS and FID values are frequently worse than those of FGPS
and DSG, particularly on FFHQ. The main exception is ImageNet motion
deblurring, where DPS-MAP achieves the best LPIPS and the second-best FID;
DMPS also obtains a competitive LPIPS in this setting, although its FID
remains high. Overall, FGPS achieves the best FID in all six settings and the
best LPIPS in five of them, while DSG is generally the next strongest
perceptual baseline. Our method does not dominate these perceptual metrics:
its results often lie between the perception-oriented baselines and the
strong distortion-oriented solvers, consistent with its emphasis on
measurement-faithful recovery. These observations motivate reporting LPIPS
and FID alongside, rather than in place of, the paired PSNR and SSIM results.

\begin{table*}[t]
\centering
\caption{
Complementary perceptual comparison on FFHQ and ImageNet under three
deblurring settings. Lower LPIPS and FID are better. Bold and underlined
values indicate the best and second-best results, respectively.
}
\label{tab:perceptual_metrics}

\scriptsize
\setlength{\tabcolsep}{4.0pt}
\renewcommand{\arraystretch}{1.08}

\begin{tabular}{lcccccc}
\toprule
& \multicolumn{2}{c}{Gaussian Blur}
& \multicolumn{2}{c}{Motion Blur}
& \multicolumn{2}{c}{Uniform Blur} \\
\cmidrule(lr){2-3}\cmidrule(lr){4-5}\cmidrule(lr){6-7}
Method
& LPIPS$\downarrow$ & FID$\downarrow$
& LPIPS$\downarrow$ & FID$\downarrow$
& LPIPS$\downarrow$ & FID$\downarrow$ \\
\midrule
\multicolumn{7}{c}{\textbf{FFHQ Dataset}} \\
\midrule
DPS~\citep{chung2023dps}
& 0.150 & 31.161
& 0.197 & 34.278
& 0.175 & 32.164 \\
DSG~\citep{yang2023dsg}
& \underline{0.104} & \underline{26.155}
& \underline{0.130} & \underline{28.557}
& \underline{0.134} & \underline{28.609} \\
FGPS~\citep{thaker2025fgps}
& \textbf{0.093} & \textbf{23.593}
& \textbf{0.103} & \textbf{24.670}
& \textbf{0.116} & \textbf{26.136} \\
DPS-MAP~\citep{li2025efficient}
& 0.236 & 61.684
& 0.195 & 67.177
& 0.241 & 69.985 \\
DMPS~\citep{meng2025dmps}
& 0.220 & 62.453
& 0.188 & 69.320
& 0.231 & 71.458 \\
Ours
& 0.194 & 49.707
& 0.153 & 45.754
& 0.176 & 50.993 \\
\midrule
\multicolumn{7}{c}{\textbf{ImageNet Dataset}} \\
\midrule
DPS
& 0.336 & 48.369
& 0.380 & 54.001
& 0.330 & 47.078 \\
DSG
& \underline{0.285} & \underline{41.534}
& 0.337 & 51.545
& \underline{0.319} & \underline{45.410} \\
FGPS
& \textbf{0.210} & \textbf{33.502}
& \underline{0.255} & \textbf{38.079}
& \textbf{0.270} & \textbf{40.035} \\
DPS-MAP
& 0.378 & 66.920
& \textbf{0.247} & \underline{46.993}
& 0.336 & 62.013 \\
DMPS
& 0.376 & 59.277
& 0.287 & 74.868
& 0.354 & 80.468 \\
Ours
& 0.394 & 63.909
& 0.291 & 56.234
& 0.332 & 56.722 \\
\bottomrule
\end{tabular}
\end{table*}

\subsection{Additional Ablation Results}
\label{app:additional_ablation}

\paragraph{Complete Core-Component Ablation}

\begin{table}[t]
\centering
\caption{
Complete component ablation on the fixed 50-image FFHQ motion-deblurring
subset using a single DDPM branch.
$\Delta$PSNR is measured relative to the full configuration.
}
\label{tab:complete_component_ablation}
\small
\setlength{\tabcolsep}{4pt}
\renewcommand{\arraystretch}{1.08}
\begin{tabular}{lrrrr}
\toprule
Variant
& PSNR$\uparrow$
& $\Delta$PSNR
& SSIM$\uparrow$
& LPIPS$\downarrow$ \\
\midrule
Full
& 29.133 & -- & 0.8368 & 0.1322 \\
w/o attenuation weighting
& 29.138 & +0.005 & 0.8367 & 0.1318 \\
Fixed $G_t$
& 29.133 & +0.000 & 0.8368 & 0.1325 \\
w/o spectral prior
& 29.084 & -0.049 & 0.8366 & 0.1342 \\
w/o trajectory regularization
& 24.446 & -4.687 & 0.5659 & 0.3243 \\
w/o frequency decoupling
& 29.084 & -0.049 & 0.8365 & 0.1343 \\
\bottomrule
\end{tabular}
\end{table}

The complete ablation results are reported in Table~\ref{tab:complete_component_ablation}.
Removing the spectral prior or the complete frequency-decoupling mechanism
reduces PSNR by about $0.05$ dB, while removing the trajectory regularizer causes a substantially larger degradation. By contrast, replacing $W_{\kappa}$ with uniform weighting or fixing $G_t$ produces only minor changes on this setting.

\paragraph{Prior Unification and Motion Simplification}

Earlier motion configurations used a correction-space prior together
with a late auxiliary correction. In the final formulation, we use the
same clean-image spectral $\ell_1$ prior as in the other deblurring
settings and remove the auxiliary correction.

\begin{table}[t]
\centering
\caption{
Effect of the unified spectral prior and simplified motion configuration
on the full 1000-image evaluation sets.
}
\label{tab:prior_unification}
\small
\setlength{\tabcolsep}{4pt}
\renewcommand{\arraystretch}{1.08}
\begin{tabular}{llccc}
\toprule
Dataset & Configuration
& PSNR$\uparrow$
& SSIM$\uparrow$
& LPIPS$\downarrow$ \\
\midrule
FFHQ
& Previous
& 29.404 & 0.8380 & 0.1569 \\
FFHQ
& Unified
& 29.447 & 0.8392 & 0.1534 \\
\midrule
ImageNet
& Previous
& 26.382 & 0.7606 & 0.2948 \\
ImageNet
& Unified
& 26.415 & 0.7627 & 0.2913 \\
\bottomrule
\end{tabular}
\end{table}

As shown in Table~\ref{tab:prior_unification}, the unified configuration improves all three metrics on both datasets. The PSNR gain is $0.044$ dB on FFHQ and $0.032$ dB on ImageNet, while also improving SSIM and LPIPS. These results support using a common clean-image spectral prior across the three blur operators.

\paragraph{Operator-Specific CGJ Correction}

\begin{table}[t]
\centering
\caption{
Effect of the operator-specific CGJ correction on FFHQ uniform
deblurring over the full 1000-image evaluation set.
}
\label{tab:cgj_ablation}
\small
\setlength{\tabcolsep}{5pt}
\renewcommand{\arraystretch}{1.08}
\begin{tabular}{lccc}
\toprule
Variant
& PSNR$\uparrow$
& SSIM$\uparrow$
& LPIPS$\downarrow$ \\
\midrule
Base two-branch
& 29.108 & 0.8252 & 0.1836 \\
CGJ two-branch
& 29.160 & 0.8268 & 0.1757 \\
\bottomrule
\end{tabular}
\end{table}

As shown in Table~\ref{tab:cgj_ablation}, the CGJ two-branch variant increases PSNR from $29.108$ to $29.160$ dB and improves both SSIM and LPIPS. The PSNR improvement is $0.052$ dB, with higher PSNR on 919 of the 1000 paired reconstructions.

\subsection{Analysis of Two-Branch Inference}
\label{app:consensus}

Our final reconstruction combines two independently completed sampling trajectories by a fixed output-space weighted fusion. For Gaussian and motion deblurring, the two branches are DDPM and DDIM. For uniform deblurring, the DDPM branch is replaced by the CGJ-DDPM branch. The fusion weights are fixed for each task and are applied to the final reconstructed images rather than during the reverse diffusion process.

\begin{table*}[!t]
\centering
\caption{
Single-branch and two-branch reconstruction results on the full
1000-image evaluation sets. Each entry reports PSNR / SSIM / LPIPS. The two fusion weights correspond to the first and second branch, respectively. For uniform deblurring, the Ours-DDPM column corresponds to the CGJ-DDPM branch.
}
\label{tab:two_track_analysis}
\small
\setlength{\tabcolsep}{4.5pt}
\renewcommand{\arraystretch}{1.08}
\resizebox{\textwidth}{!}{%
\begin{tabular}{llcccc}
\toprule
Dataset & Task
& Ours-DDPM
& Ours-DDIM
& Ours
& Fusion weights \\
\midrule

FFHQ & Gaussian
& 29.095 / 0.8308 / 0.1741
& 29.236 / 0.8340 / 0.2127
& 29.528 / 0.8423 / 0.1943
& 0.50 / 0.50 \\

FFHQ & Motion
& 28.954 / 0.8319 / 0.1378
& 28.799 / 0.8168 / 0.1893
& 29.447 / 0.8392 / 0.1534
& 0.50 / 0.50 \\

FFHQ & Uniform
& 28.775 / 0.8187 / 0.1621
& 28.216 / 0.8025 / 0.2011
& 29.160 / 0.8268 / 0.1757
& 0.50 / 0.50 \\

\midrule

ImageNet & Gaussian
& 25.463 / 0.7234 / 0.3867
& 25.333 / 0.7148 / 0.4103
& 25.500 / 0.7249 / 0.3936
& 0.70 / 0.30 \\

ImageNet & Motion
& 26.256 / 0.7570 / 0.2889
& 26.133 / 0.7509 / 0.2995
& 26.415 / 0.7627 / 0.2913
& 0.50 / 0.50 \\

ImageNet & Uniform
& 25.593 / 0.7165 / 0.3204
& 25.193 / 0.6959 / 0.3972
& 25.869 / 0.7240 / 0.3324
& 0.55 / 0.45 \\

\bottomrule
\end{tabular}
}
\end{table*}

As shown in Table~\ref{tab:two_track_analysis}, the two-branch
reconstruction improves both PSNR and SSIM over either constituent branch in all six settings. Relative to the better single branch, the PSNR gains are $0.292$, $0.493$, and $0.385$ dB on FFHQ Gaussian, motion, and uniform deblurring, and $0.037$, $0.159$, and $0.276$ dB on the corresponding ImageNet settings. The fused reconstruction also improves LPIPS relative to the DDIM branch, while the DDPM or CGJ-DDPM branch retains the lower LPIPS.

\section{{Paired Statistical Analysis}}
\label{app:statistics}

To complement the aggregate results in
Table~\ref{tab:quantitative_comparison}, we perform paired statistical comparisons on the $1{,}000$ matched reconstructions within each evaluation setting. We compare our final reconstruction with DPS-MAP and DMPS, which provide the strongest reconstruction-fidelity baselines among the compared methods. For each metric, we compute the per-image difference and report its mean together with a Student-$t$ 95\% confidence interval. We define $\Delta=\mathrm{Ours}-\mathrm{Baseline}$; therefore, positive values favor our method for PSNR and SSIM, whereas negative values favor our method for LPIPS.

\begin{table*}[!t]
\centering
\caption{
Paired statistical comparison with DPS-MAP and DMPS on the full
1000-image evaluation sets.
Each entry reports the mean per-image difference and its Student-$t$
95\% confidence interval.
We define $\Delta=\mathrm{Ours}-\mathrm{Baseline}$; positive values
favor our method for PSNR and SSIM, whereas negative values favor our
method for LPIPS.
}
\label{tab:paired_statistics}

\small
\setlength{\tabcolsep}{4.0pt}
\renewcommand{\arraystretch}{1.08}

\resizebox{\textwidth}{!}{%
\begin{tabular}{lllccc}
\toprule
Dataset & Task & Baseline
& $\Delta$PSNR
& $\Delta$SSIM
& $\Delta$LPIPS \\
\midrule

FFHQ & Gaussian & DPS-MAP
& $+0.253\,[+0.240,+0.266]$
& $+0.0058\,[+0.0055,+0.0062]$
& $-0.0418\,[-0.0428,-0.0407]$ \\

FFHQ & Gaussian & DMPS
& $+0.149\,[+0.133,+0.165]$
& $+0.0030\,[+0.0027,+0.0033]$
& $-0.0253\,[-0.0264,-0.0241]$ \\

FFHQ & Motion & DPS-MAP
& $+0.182\,[+0.169,+0.196]$
& $+0.0051\,[+0.0046,+0.0057]$
& $-0.0418\,[-0.0430,-0.0407]$ \\

FFHQ & Motion & DMPS
& $+0.221\,[+0.199,+0.243]$
& $+0.0043\,[+0.0037,+0.0048]$
& $-0.0350\,[-0.0361,-0.0338]$ \\

FFHQ & Uniform & DPS-MAP
& $+0.302\,[+0.286,+0.318]$
& $+0.0101\,[+0.0096,+0.0106]$
& $-0.0652\,[-0.0666,-0.0637]$ \\

FFHQ & Uniform & DMPS
& $+0.224\,[+0.209,+0.239]$
& $+0.0079\,[+0.0074,+0.0083]$
& $-0.0551\,[-0.0564,-0.0537]$ \\

\midrule

ImageNet & Gaussian & DPS-MAP
& $+0.069\,[+0.046,+0.092]$
& $+0.0027\,[+0.0014,+0.0040]$
& $+0.0156\,[+0.0139,+0.0174]$ \\

ImageNet & Gaussian & DMPS
& $+0.306\,[+0.283,+0.330]$
& $+0.0141\,[+0.0132,+0.0149]$
& $+0.0174\,[+0.0156,+0.0192]$ \\

ImageNet & Motion & DPS-MAP
& $-0.243\,[-0.266,-0.220]$
& $-0.0052\,[-0.0066,-0.0037]$
& $+0.0445\,[+0.0424,+0.0465]$ \\

ImageNet & Motion & DMPS
& $+0.090\,[+0.069,+0.112]$
& $+0.0118\,[+0.0103,+0.0133]$
& $+0.0046\,[+0.0020,+0.0071]$ \\

ImageNet & Uniform & DPS-MAP
& $-0.016\,[-0.042,+0.011]$
& $+0.0018\,[+0.0006,+0.0030]$
& $-0.0039\,[-0.0060,-0.0017]$ \\

ImageNet & Uniform & DMPS
& $+0.277\,[+0.252,+0.303]$
& $+0.0149\,[+0.0134,+0.0163]$
& $-0.0211\,[-0.0240,-0.0181]$ \\

\bottomrule
\end{tabular}
}
\end{table*}

As shown in Table~\ref{tab:paired_statistics}, the paired results are consistent with the aggregate comparison. On FFHQ, our method improves PSNR, SSIM, and LPIPS over both DPS-MAP and DMPS for all three blur operators, with all corresponding confidence intervals excluding zero. On ImageNet Gaussian deblurring, our method improves PSNR and SSIM over both baselines, while LPIPS shows a fidelity--perception trade-off. For ImageNet motion deblurring, DPS-MAP remains stronger across all three metrics, whereas our method improves PSNR and SSIM over DMPS but yields a slightly higher LPIPS. For ImageNet uniform deblurring, the PSNR difference relative to DPS-MAP is small and its confidence interval includes zero, while SSIM and LPIPS favor our method. Relative to DMPS, our method improves all three metrics.

\section{Sensitivity}
\label{app:sensitivity}

We study the sensitivity of the proposed guidance to the attenuation parameter $\rho$ and the spectral-prior strength on the fixed FFHQ motion-deblurring Dev50 set. In each experiment, only the parameter under study is varied, while all remaining settings are kept fixed. %

\subsection{Sensitivity to $\rho$}

We vary the attenuation parameter over
$\rho\in\{10^{-4},3\times10^{-4},10^{-3},3\times10^{-3},10^{-2}\}$,
with $\rho=10^{-3}$ as the default value.

\begin{figure}[t]
    \centering
    \includegraphics[width=0.9\textwidth]
    {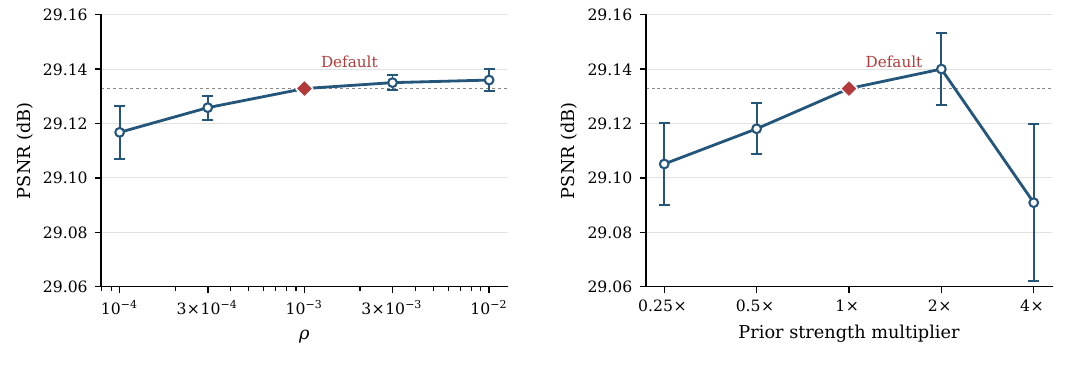}
    \caption{
    Hyperparameter sensitivity on the FFHQ motion-deblurring Dev50 set.
    Left: sensitivity to the attenuation parameter $\rho$, with
    $\rho=10^{-3}$ as the default.
    Right: sensitivity to the spectral-prior strength, shown as a multiplier
    of the default $\lambda_{\mathrm{prior}}=0.005$.
    In both panels, all other parameters are fixed to the default
    configuration. Error bars denote paired Student-$t$ 95\% confidence
    intervals of the PSNR difference relative to the default setting,
    shifted to the corresponding absolute PSNR levels.
    }
    \label{fig:sensitivity}
\end{figure}

As shown in the left panel of Fig.~\ref{fig:sensitivity}, the reconstruction quality remains stable around the default value. Reducing $\rho$ below $10^{-3}$ gradually decreases PSNR. Increasing $\rho$ to $3\times10^{-3}$ or $10^{-2}$ changes the mean PSNR by only $+0.002$ and $+0.003$ dB, respectively, and both paired confidence intervals include zero. SSIM changes only slightly over these larger values, while LPIPS shows a small improvement. Overall, the results indicate limited sensitivity to $\rho$ around the default configuration.

\subsection{Sensitivity to the Prior Strength}

We next vary the spectral-prior coefficient using multipliers
$\{0.25,0.5,1,2,4\}$ relative to the default
$\lambda_{\mathrm{prior}}=0.005$.

The right panel of Fig.~\ref{fig:sensitivity} shows that weakening the prior to $0.25\times$ or $0.5\times$ reduces PSNR. Increasing the prior to $2\times$ gives a small mean PSNR increase of $0.007$ dB and improves LPIPS, while the paired PSNR confidence interval still includes zero and SSIM decreases slightly. At $4\times$, LPIPS further decreases from $0.1322$ to $0.1279$, whereas PSNR decreases by $0.042$ dB and SSIM by $0.0020$, showing a clearer fidelity--perception trade-off.
Together with the $\rho$ experiment, these results show that the
default configuration lies in a relatively stable operating region.

\section{Additional Qualitative Results}
\label{app:qualitative}

We provide additional qualitative comparisons on FFHQ and ImageNet
for Gaussian, motion, and uniform deblurring. All results are taken
from the frozen Full1000 reconstructions at
$\sigma_{\mathrm{y}}=0.05$. In addition to the full-image comparisons,
we provide enlarged crops, reconstruction-error visualizations, and
frequency-domain visualizations of the proposed guidance mechanism.

\subsection{Additional FFHQ Reconstructions}

Figure~\ref{fig:supp_ffhq_combined} shows additional FFHQ reconstruction results for Gaussian, motion, and uniform deblurring. The examples include a range of facial structures, hair regions, glasses, and other fine details that are sensitive to blur and reconstruction artifacts.

\begin{figure*}[t]
    \centering
    \includegraphics[width=\textwidth]
    {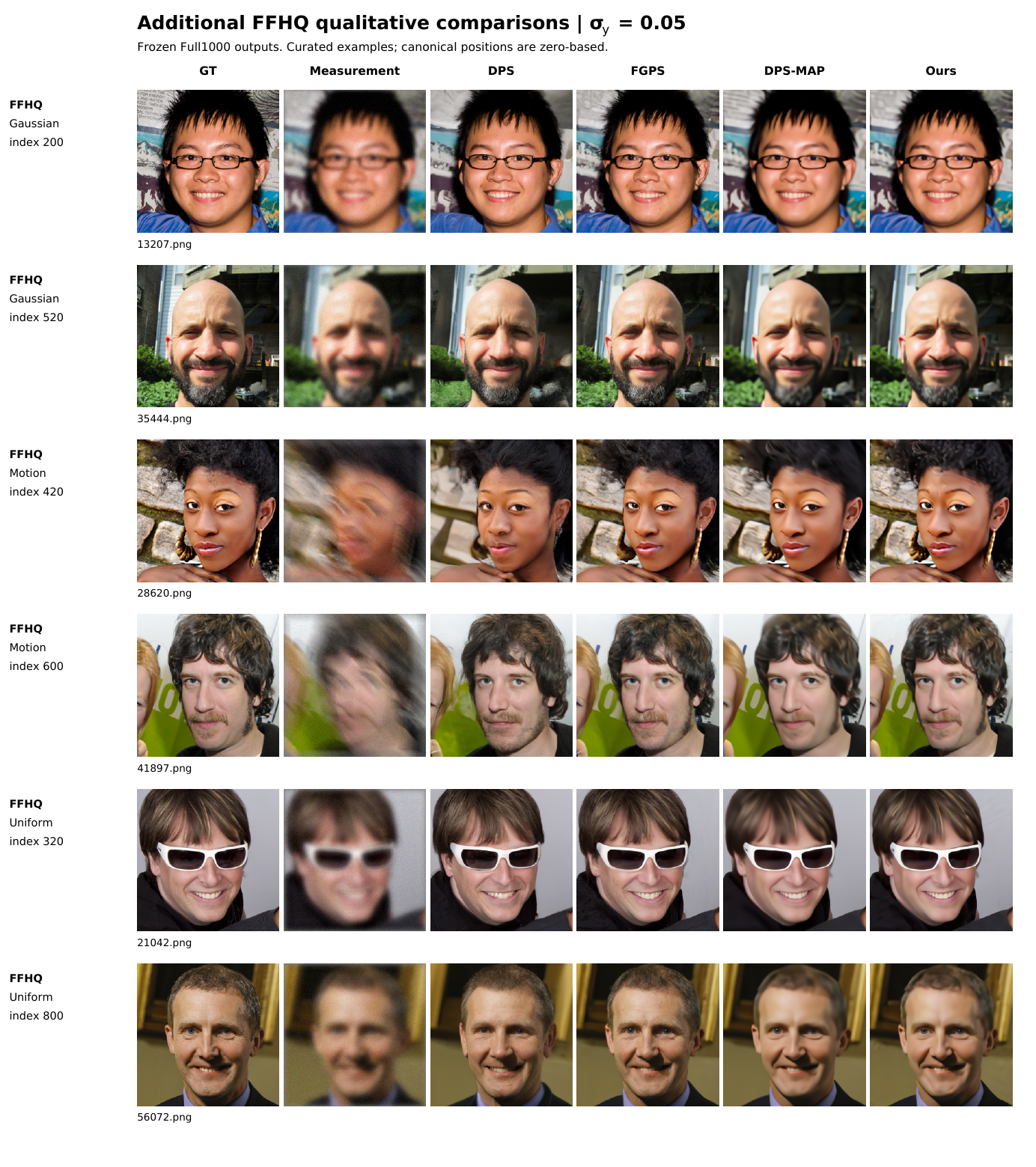}
    \caption{
    Additional qualitative comparisons on FFHQ under Gaussian, motion,
    and uniform deblurring at $\sigma_{\mathrm{y}}=0.05$.
    Two additional Full1000 examples are shown for each degradation setting.
    Columns show the ground truth, measurement, DPS, FGPS, DPS-MAP,
    and our method.
    }
    \label{fig:supp_ffhq_combined}
\end{figure*}

Across the three degradation settings, the methods differ most clearly
around fine facial structures and high-frequency boundaries.
The additional examples complement the quantitative results by showing
these differences over a broader set of FFHQ images.

\subsection{Additional ImageNet Reconstructions}

Figure~\ref{fig:supp_imagenet_combined} shows additional ImageNet reconstruction results for Gaussian, motion, and uniform deblurring. Compared with FFHQ, these examples contain more diverse object structures, textures, and high-frequency boundaries, providing complementary qualitative comparisons across different image contents.

\begin{figure*}[t]
    \centering
    \includegraphics[width=\textwidth]
    {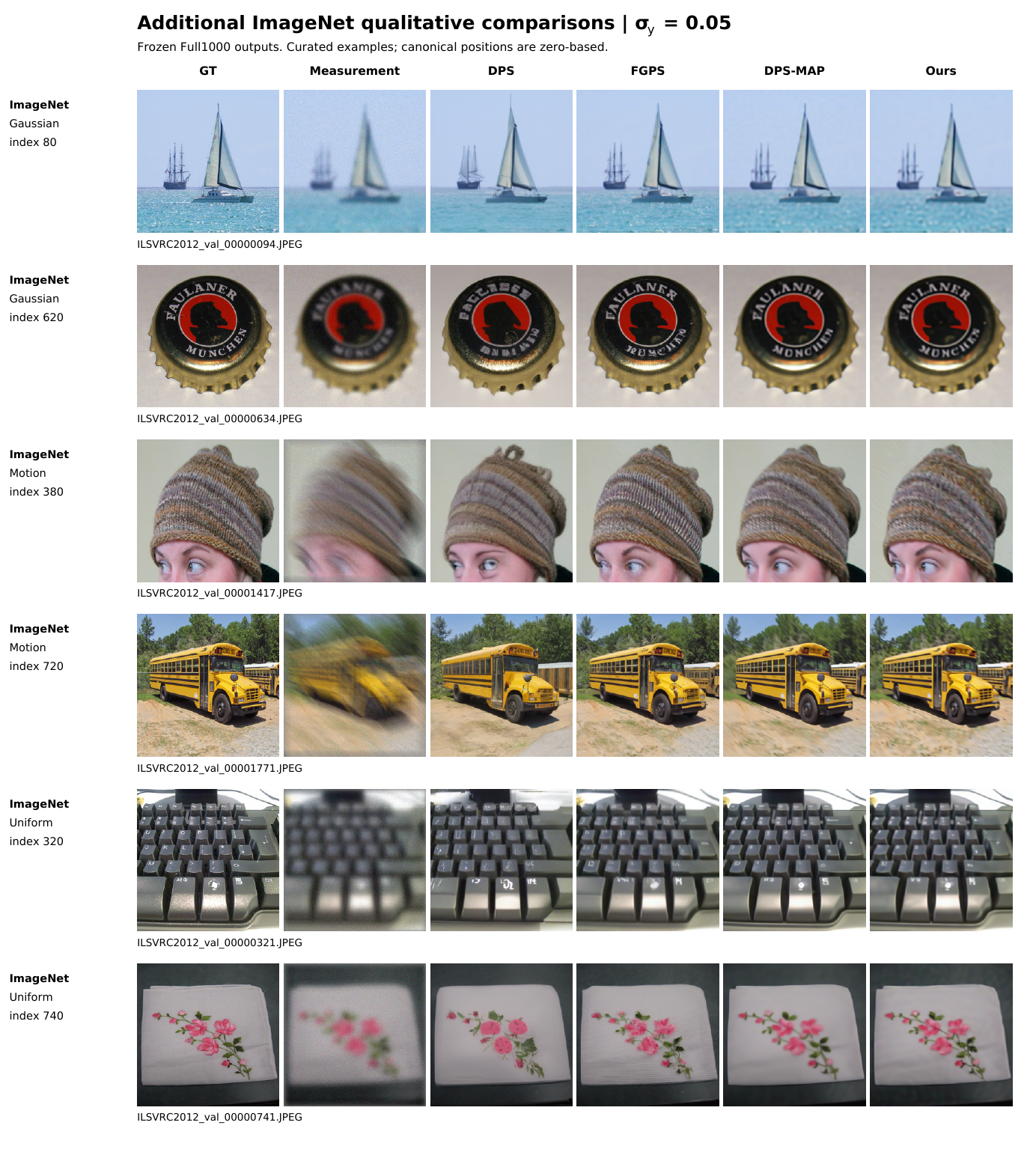}
    \caption{
    Additional qualitative comparisons on ImageNet under Gaussian, motion,
    and uniform deblurring at $\sigma_{\mathrm{y}}=0.05$.
    Two additional Full1000 examples are shown for each degradation setting.
    Columns show the ground truth, measurement, DPS, FGPS, DPS-MAP,
    and our method.
    }
    \label{fig:supp_imagenet_combined}
\end{figure*}

The ImageNet examples further illustrate the differences among the methods on thin structures, repetitive textures, and object boundaries. Together with the FFHQ results, they provide qualitative comparisons across both face images and more diverse natural-image content.

\subsection{Zoomed Comparisons}

To examine local reconstruction details more closely,
Figs.~\ref{fig:supp_zoom_ffhq} and~\ref{fig:supp_zoom_imagenet}
show enlarged crops from the FFHQ and ImageNet results, respectively.
For each example, the same crop coordinates are used for all methods.

\begin{figure*}[t]
    \centering
    \includegraphics[width=\textwidth]
    {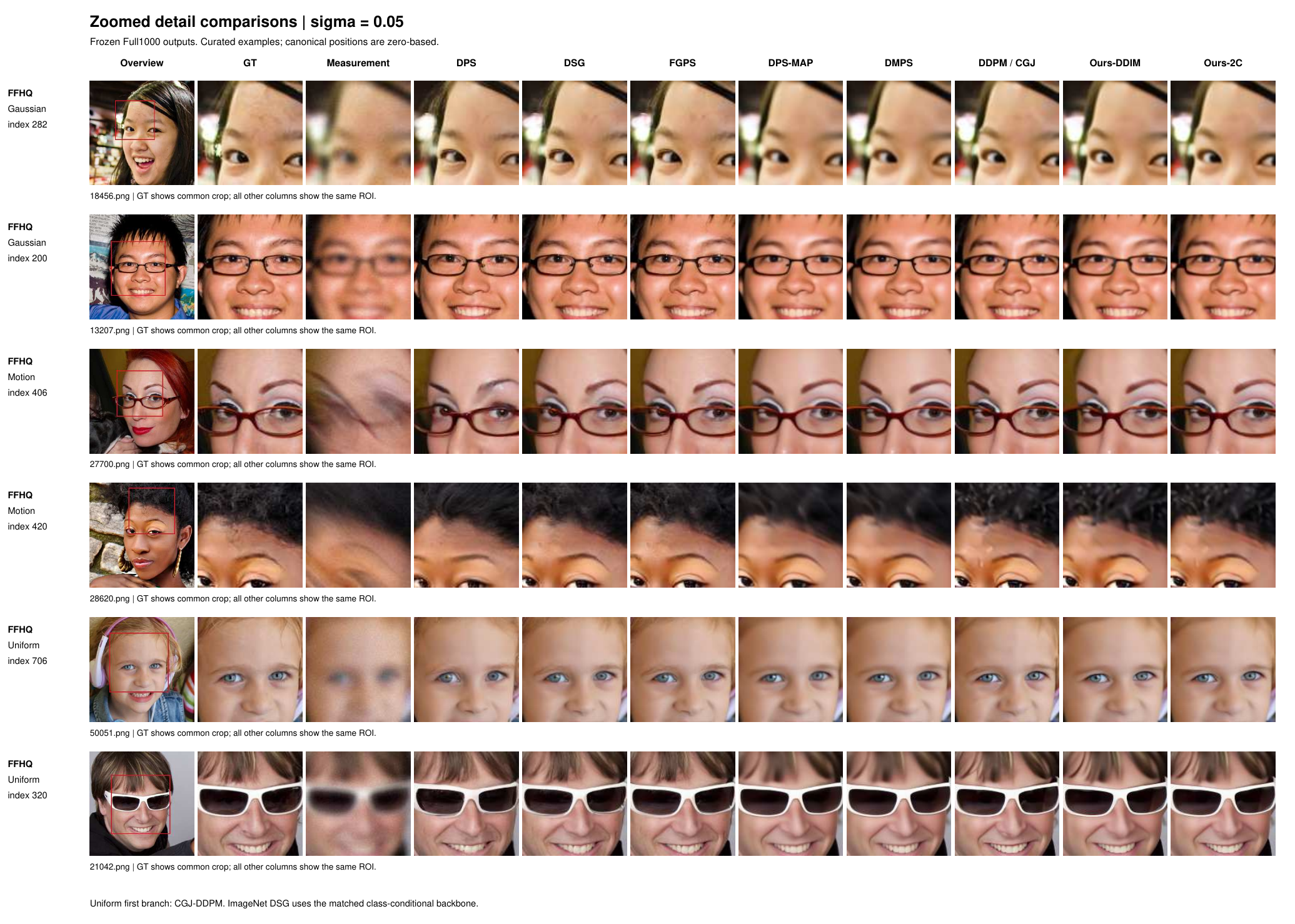}
    \caption{
    Zoomed comparisons on FFHQ at $\sigma_{\mathrm{y}}=0.05$.
    The examples cover Gaussian, motion, and uniform deblurring.
    For each row, all methods are shown using the same spatial crop.
    The uniform-deblurring comparison uses CGJ-DDPM as the first Ours
    branch.
    }
    \label{fig:supp_zoom_ffhq}
\end{figure*}

\begin{figure*}[t]
    \centering
    \includegraphics[width=\textwidth]
    {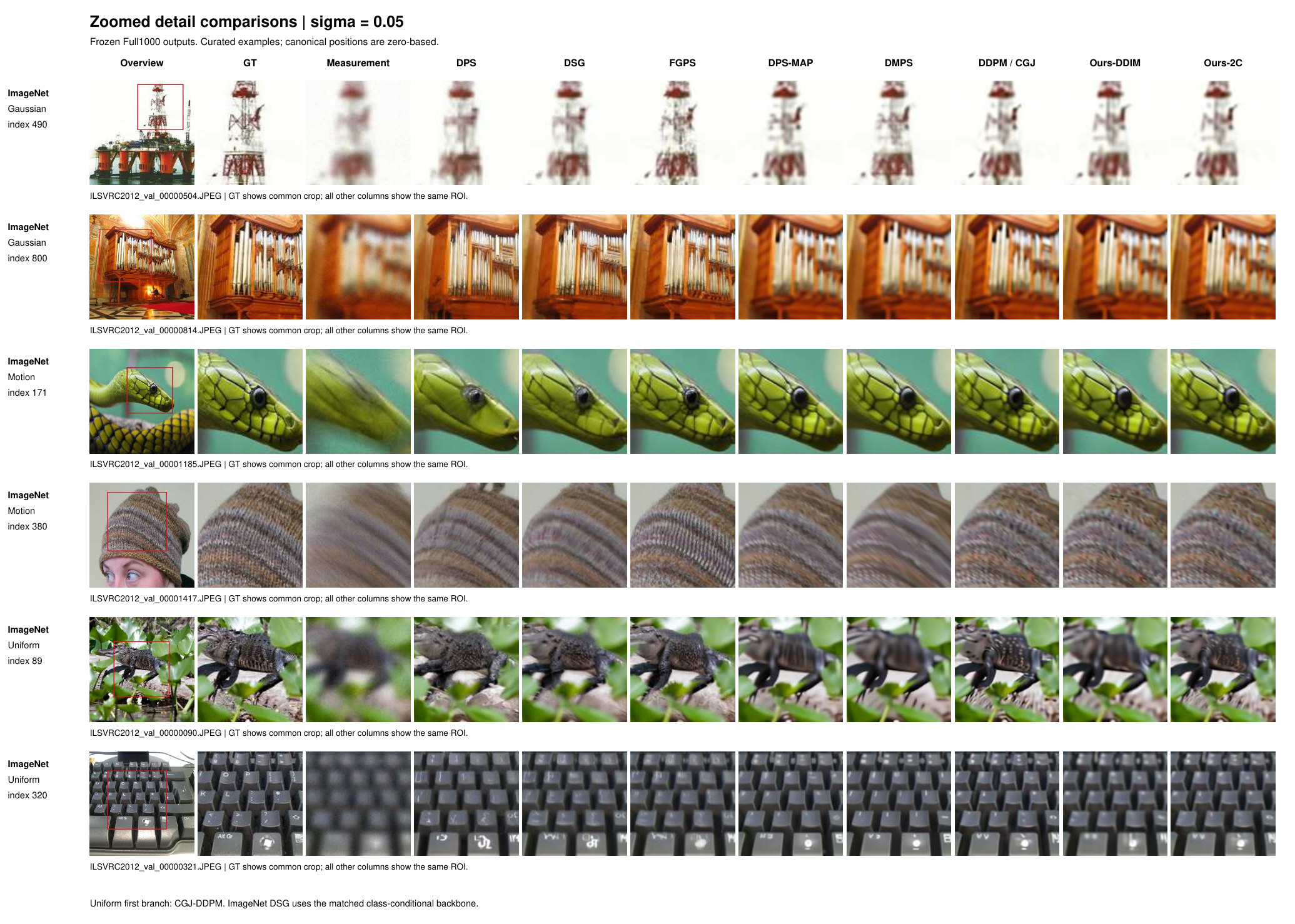}
    \caption{
    Zoomed comparisons on ImageNet at $\sigma_{\mathrm{y}}=0.05$.
    The examples cover Gaussian, motion, and uniform deblurring.
    For each row, all methods are shown using the same spatial crop.
    For uniform deblurring, the first Ours branch is CGJ-DDPM.
    }
    \label{fig:supp_zoom_imagenet}
\end{figure*}

The enlarged views make differences in fine edges, thin structures,
and local textures easier to inspect than in the full-resolution
comparisons alone. They also show the behavior of the two individual
sampling branches before the final output-space fusion.

\subsection{Absolute Reconstruction Error Visualization}

To further inspect the reconstruction behavior of different methods,
we provide channel-mean absolute error visualizations on representative
FFHQ and ImageNet examples. Figures~\ref{fig:supp_absolute_error_ffhq}
and~\ref{fig:supp_absolute_error_imagenet} show the corresponding error
maps. For a reconstruction $\widehat{\mathrm{x}}$, the channel-mean
absolute reconstruction error at pixel $(i,j)$ is defined as
\begin{equation}
    \mathrm{E}_{\mathrm{abs}}(i,j)
    =
    \frac{1}{3}
    \sum_{c=1}^{3}
    \left|
        \widehat{\mathrm{x}}_{i,j,c}
        -
        \mathrm{x}_{\mathrm{GT},i,j,c}
    \right|.
\end{equation}
All methods in the same example are visualized using the same display
scale so that the error magnitude can be compared directly.

\begin{figure*}[t]
    \centering
    \includegraphics[width=\textwidth]
    {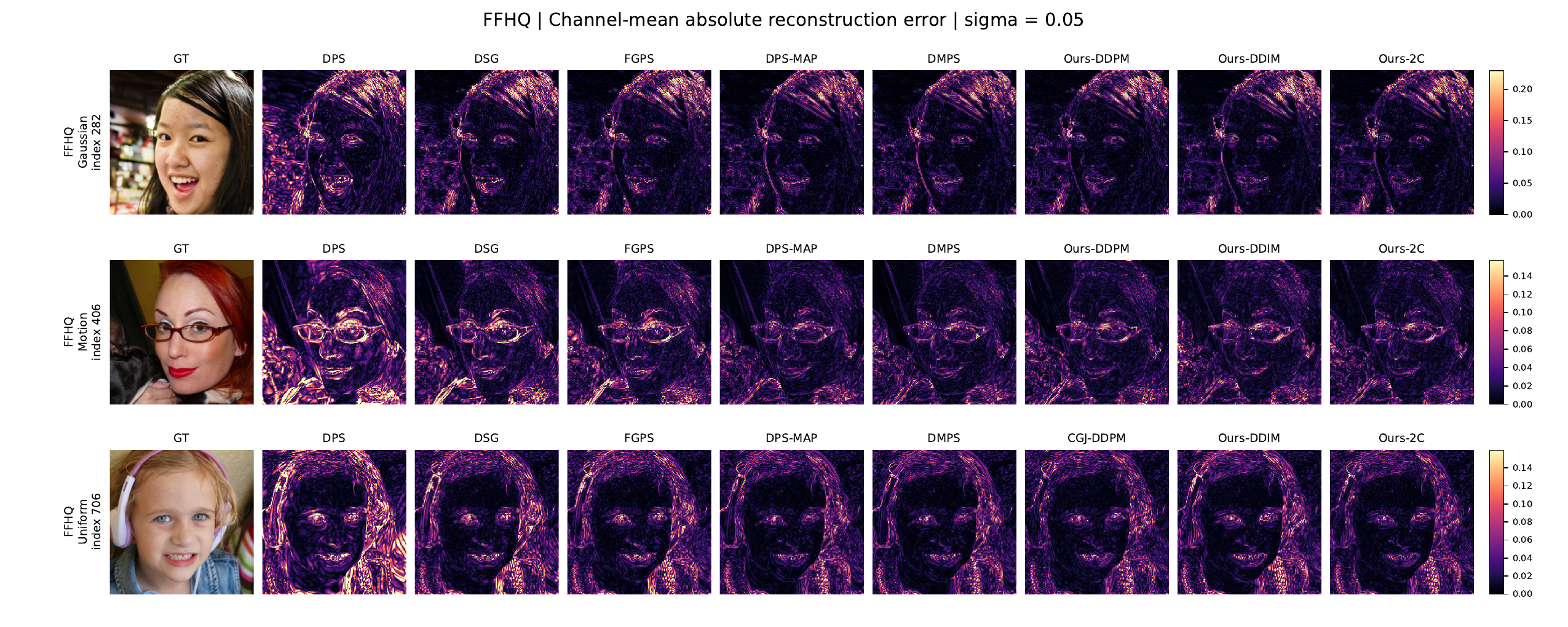}
    \caption{
    Absolute reconstruction error on representative FFHQ examples.
    Brighter regions indicate larger pixel-wise deviation from the
    ground truth.
    All methods are displayed using the same error scale within each
    example.
    }
    \label{fig:supp_absolute_error_ffhq}
\end{figure*}

\begin{figure*}[t]
    \centering
    \includegraphics[width=\textwidth]
    {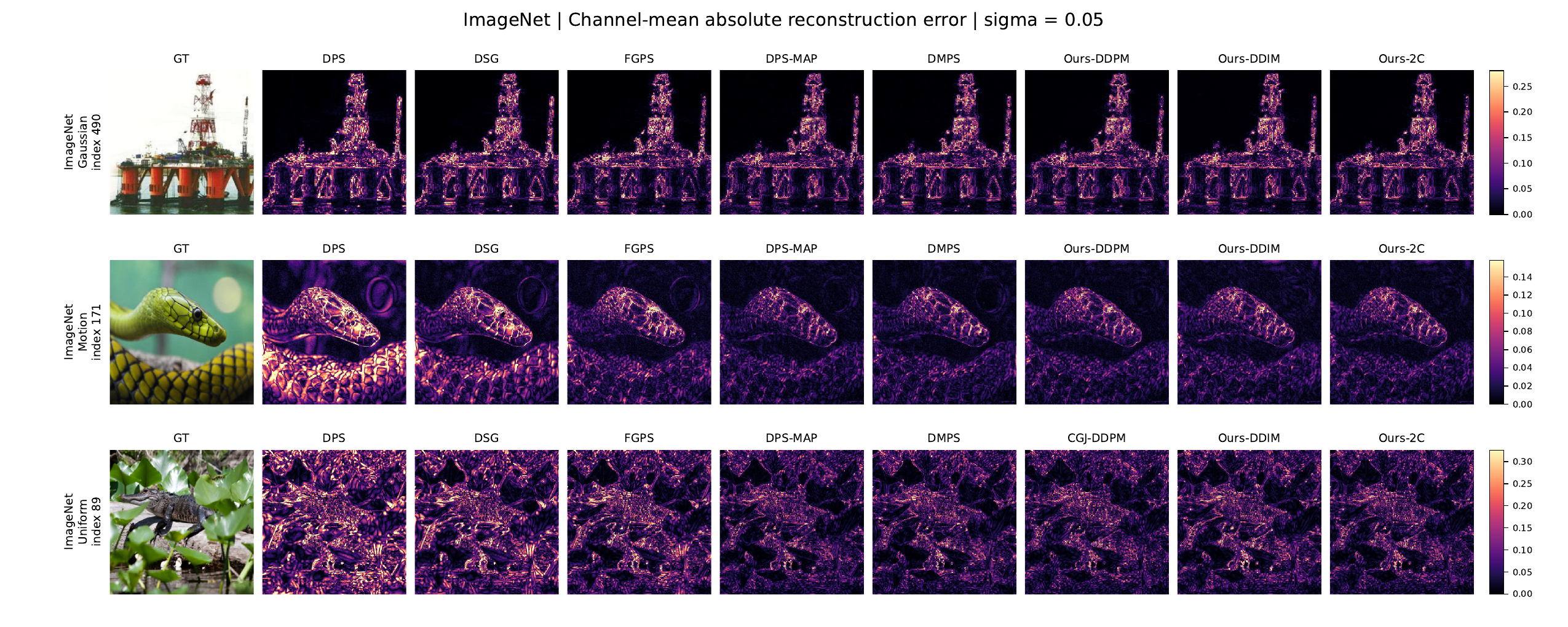}
    \caption{
    Absolute reconstruction error on representative ImageNet examples.
    Brighter regions indicate larger pixel-wise deviation from the
    ground truth.
    All methods are displayed using the same error scale within each
    example.
    }
    \label{fig:supp_absolute_error_imagenet}
\end{figure*}

The error maps provide a complementary view to the qualitative
comparisons in the previous subsection. The absolute error maps highlight where reconstruction discrepancies are concentrated. In both FFHQ and ImageNet examples, the final Ours output typically exhibits more spatially localized errors than the single-branch reconstruction, especially around salient edges and fine structures. This behavior is consistent with the quantitative improvements of the two-branch reconstruction reported in Table~\ref{tab:two_track_analysis}. At the same time, the visualizations also reflect the different fidelity--perception trade-offs of the compared baselines, which is consistent with the metric trends in the main paper.

\subsection{Visualization of Frequency-Decoupled Guidance}

We finally visualize the two frequency-domain components that define the proposed guidance: the kernel-dependent attenuation weight $W_{\kappa}$ and the time-dependent masks $G_t$ and $M_t$.
These visualizations are computed directly from the frozen kernels and sampling configurations used in the experiments.

\begin{figure*}[t]
    \centering
    \includegraphics[width=0.78\textwidth]
    {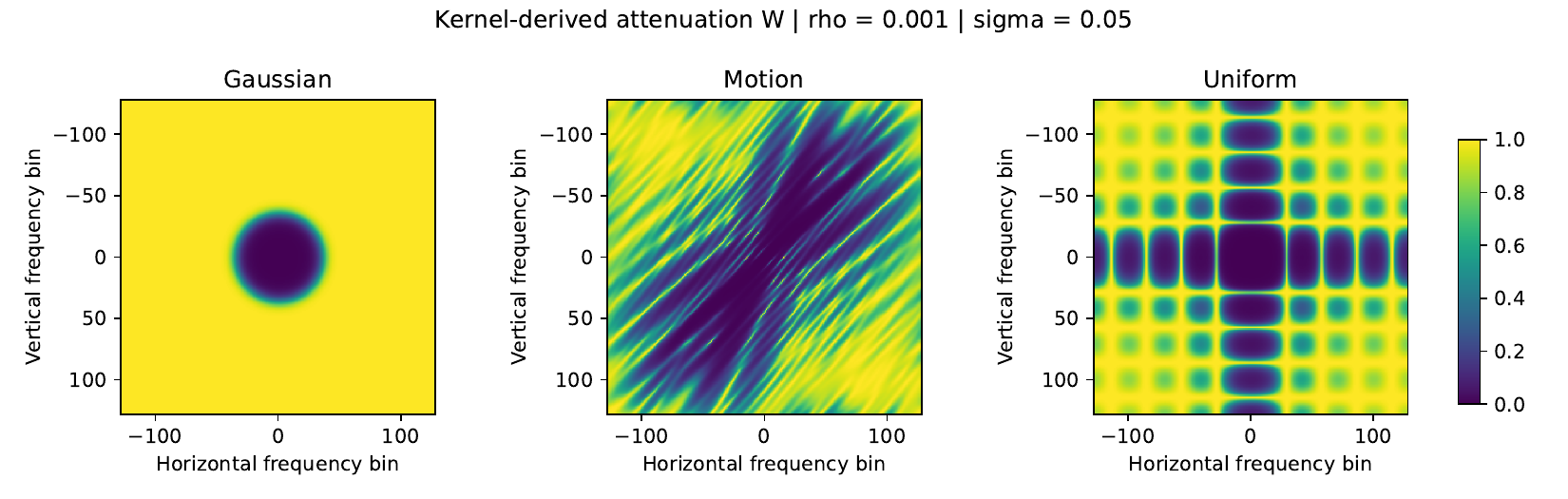}
    \caption{
    Kernel-dependent spectral attenuation weights $W_{\kappa}$ for
    Gaussian, motion, and uniform blur with $\rho=10^{-3}$.
    Larger values correspond to frequencies that are more strongly
    attenuated by the blur kernel.
    The Gaussian kernel produces a radially structured attenuation
    pattern, the motion kernel produces a directional pattern, and the uniform kernel produces a periodic spectral structure.
    }
    \label{fig:supp_spectral_attenuation}
\end{figure*}

Figure~\ref{fig:supp_spectral_attenuation} illustrates how the
attenuation weight depends on the blur operator. Since
$W_{\kappa}(\omega)$ increases as
$|\widehat{\kappa}(\omega)|$ decreases, strongly attenuated frequencies receive larger weights in the spectral prior. For uniform deblurring, $W_{\kappa}$ is constructed from the Fourier response of the underlying $9\times9$ blur kernel, while measurement consistency is evaluated using the implemented truncated-SVD operator.

The active measurement mask $G_t$ evolves according to the progressive frequency schedule, while
\begin{equation}
    M_t(\omega)
    =
    \bigl(1-G_t(\omega)\bigr)W_{\kappa}(\omega)
\end{equation}
restricts the attenuation weighting to frequencies outside the current active band. Figures~\ref{fig:supp_schedule_gaussian}--%
\ref{fig:supp_schedule_uniform} visualize this evolution at
$p_t\in\{0,0.25,0.5,0.75,1\}$ for the frozen branch-specific
configurations.

\begin{figure*}[t]
    \centering
    \includegraphics[width=\textwidth]
    {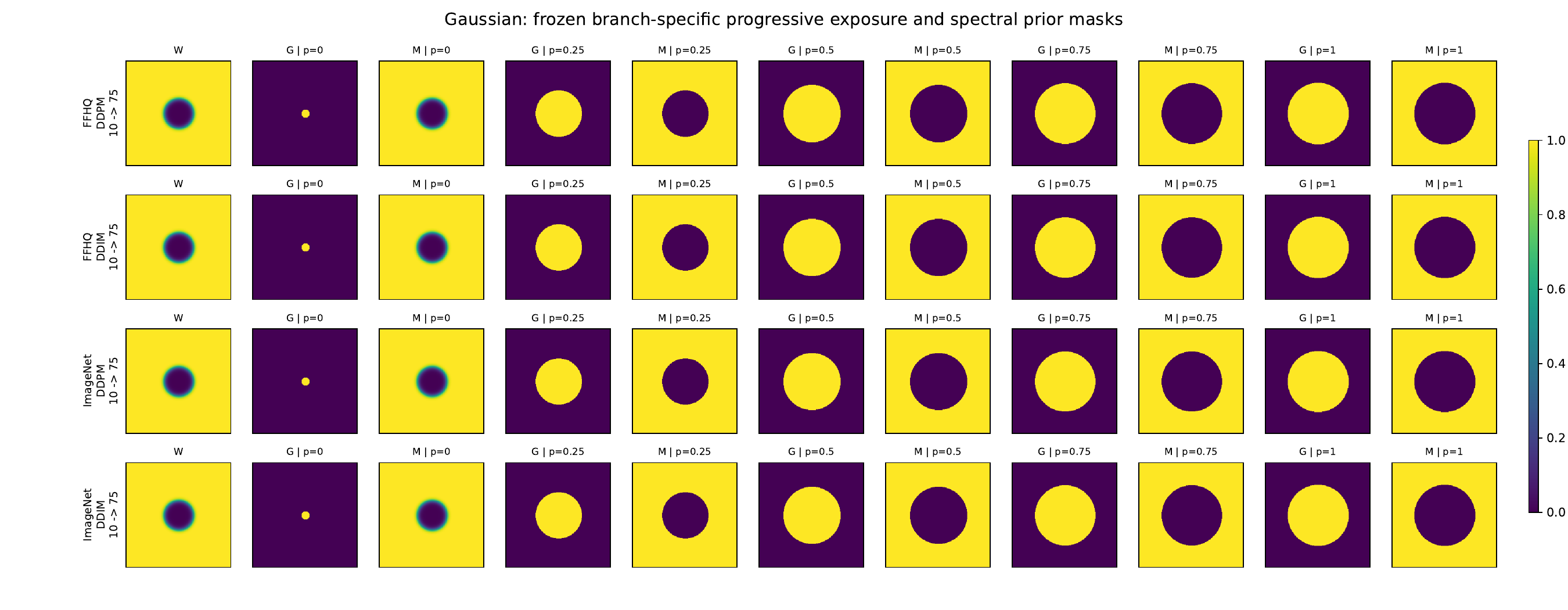}
    \caption{
    Evolution of $G_t$ and $M_t$ during reverse sampling for Gaussian
    deblurring. As the reverse process progresses, the support of
    $G_t$ expands from low to higher frequencies, while the support of
    $M_t=(1-G_t)W_{\kappa}$ correspondingly contracts.
    The rows show the frozen branch-specific schedules used for FFHQ
    and ImageNet.
    }
    \label{fig:supp_schedule_gaussian}
\end{figure*}

\begin{figure*}[t]
    \centering
    \includegraphics[width=\textwidth]
    {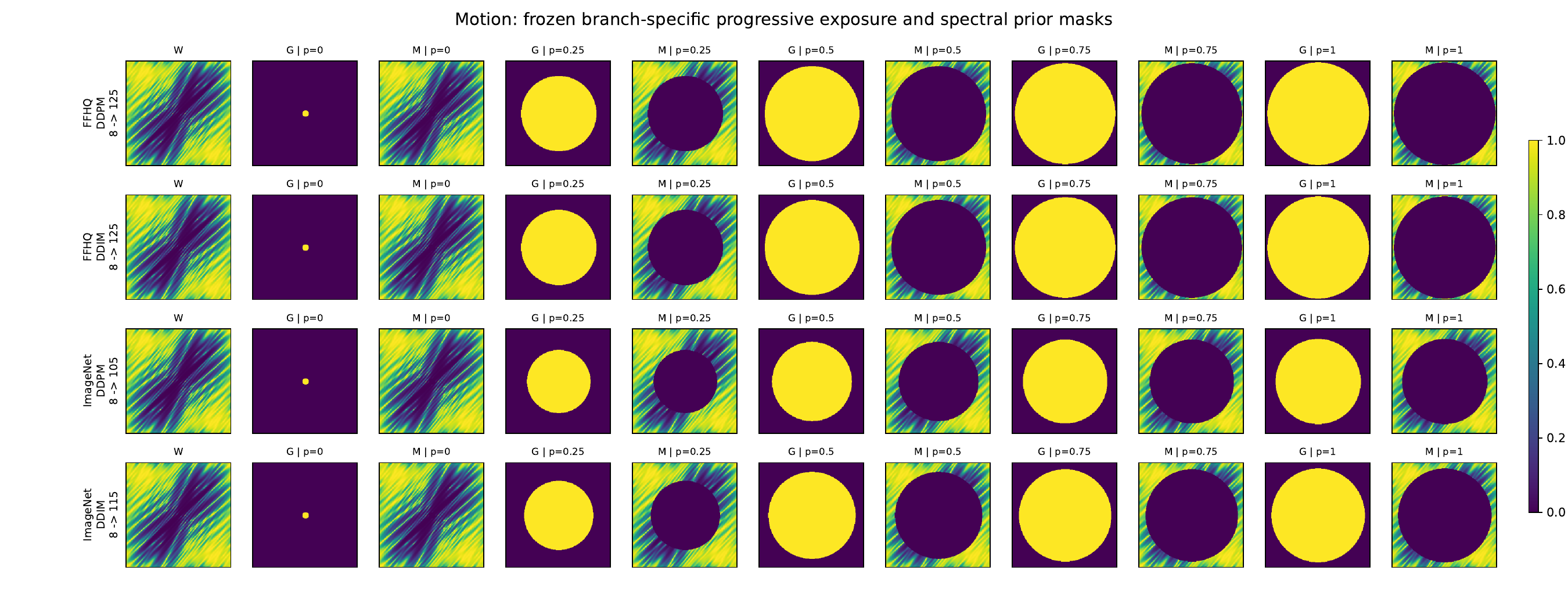}
    \caption{
    Evolution of $G_t$ and $M_t$ during reverse sampling for motion
    deblurring. The directional structure inherited from
    $W_{\kappa}$ remains visible in the inactive-band weighting
    $M_t$, while the active measurement region progressively expands. The rows correspond to the frozen branch-specific schedules used for FFHQ and ImageNet.
    }
    \label{fig:supp_schedule_motion}
\end{figure*}

\begin{figure*}[t]
    \centering
    \includegraphics[width=\textwidth]
    {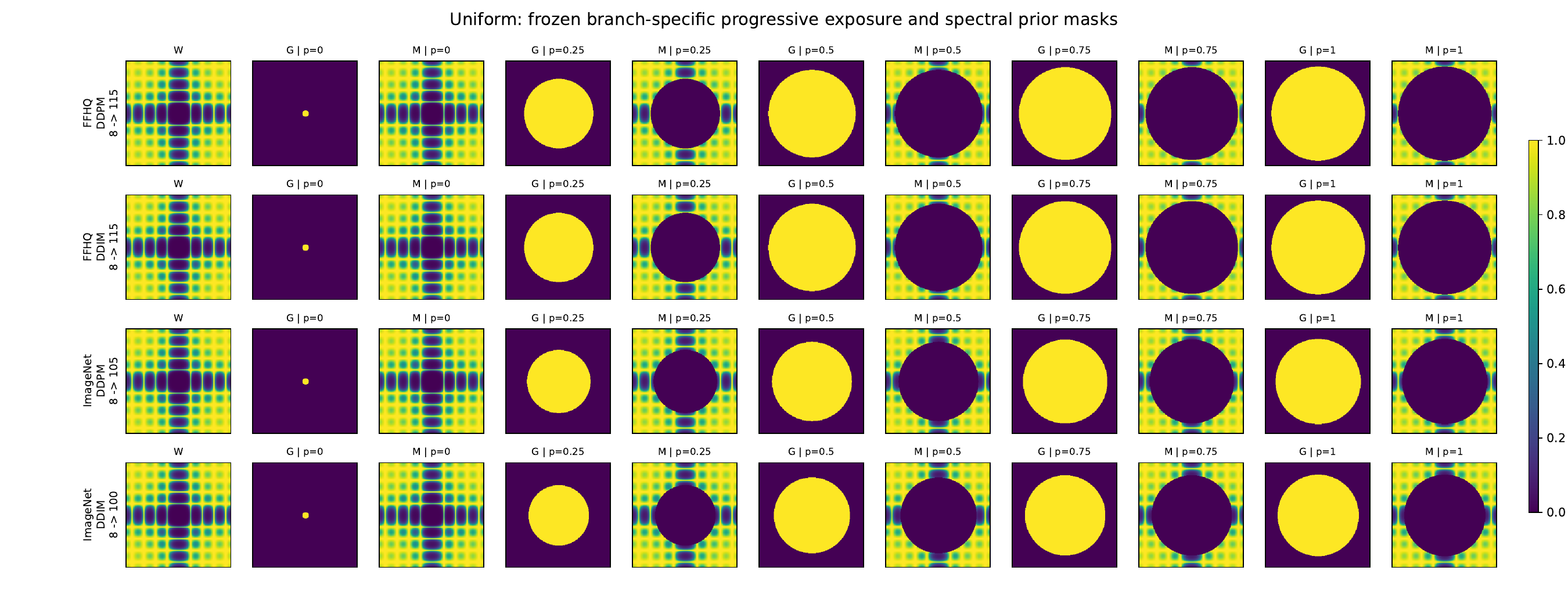}
    \caption{
    Evolution of $G_t$ and $M_t$ during reverse sampling for uniform deblurring. The periodic attenuation structure of the underlying uniform kernel modulates the inactive-frequency prior through $M_t=(1-G_t)W_{\kappa}$.
    For the uniform task, the first Ours branch is CGJ-DDPM and the
    second branch is DDIM.
    }
    \label{fig:supp_schedule_uniform}
\end{figure*}

Across the three operators, the visualizations show the two roles of the frequency-decoupled construction. The mask $G_t$ determines which frequencies currently participate in measurement consistency, whereas $W_{\kappa}$ determines how strongly the remaining inactive frequencies are weighted by the spectral prior. As the active band expands, frequencies progressively leave $M_t$ and enter the measurement-consistency term.

\section{Detailed Computational Cost}
\label{app:compute}
\paragraph{Computing infrastructure.}
Experiments were conducted using NVIDIA RTX 3090 GPUs on Linux servers.

\subsection{UNet and Operator Evaluations}

We report the inference cost in terms of the number of diffusion-model
UNet evaluations (NFE) per reconstructed image. NFE counts logical
UNet forward evaluations; gradient back-propagation through the
guidance objective is not counted as an additional NFE. Iterative
forward/adjoint operator evaluations used by conjugate-gradient
corrections are reported separately. The corresponding diffusion-model evaluation budgets are summarized in Table~\ref{tab:detailed_compute_cost}.

\begin{table*}[t]
\centering
\caption{
Number of diffusion-model UNet evaluations per reconstructed image.
The final Ours result combines two independently completed sampling
branches. For uniform deblurring, the CGJ-DDPM branch requires
additional Tweedie predictions during the operator-specific correction.
}
\label{tab:detailed_compute_cost}

\small
\setlength{\tabcolsep}{4.5pt}
\renewcommand{\arraystretch}{1.08}

\resizebox{\textwidth}{!}{%
\begin{tabular}{lcccccc}
\toprule
Method
& FFHQ-G
& FFHQ-M
& FFHQ-U
& ImageNet-G
& ImageNet-M
& ImageNet-U \\
\midrule

DPS
& 1000 & 1000 & 1000
& 1000 & 1000 & 1000 \\

DSG
& 1000 & 1000 & 1000
& 1000 & 1000 & 1000 \\

FGPS
& 1000 & 1000 & 1000
& 1000 & 1000 & 1000 \\

DPS-MAP
& 1000 & 1000 & 1000
& 100 & 100 & 100 \\

DMPS
& 100 & 100 & 100
& 100 & 1000 & 1000 \\

Ours
& 2000 & 2000 & 2150
& 2000 & 2000 & 2150 \\

\bottomrule
\end{tabular}
}
\end{table*}

For Gaussian and motion deblurring, each branch of our method uses
1000 reverse steps. The final Ours reconstruction therefore
requires 2000 UNet evaluations per image. A single non-CGJ branch uses
1000 UNet evaluations, matching the diffusion-model budget of DPS,
DSG, and FGPS.

For uniform deblurring, the CGJ correction is activated only in the
CGJ-DDPM branch. Starting from $p_t=0.25$, the correction is applied
every five reverse steps, resulting in 150 CGJ activations per image.
Each activation computes a fresh Tweedie prediction, adding 150 UNet
evaluations to the two 1000-step trajectories. The resulting total is
therefore 2150 UNet evaluations per image.

The CGJ correction also performs five conjugate-gradient iterations
per activation, corresponding to 750 CG iterations per reconstructed
image. These iterations involve applications of the implemented
forward and adjoint operators and are not counted as diffusion-model
NFE. DPS-MAP likewise uses a conjugate-gradient measurement correction,
with 20 CG iterations in its measurement-space solve.

The baseline sampling budgets are task dependent. DPS, DSG, and FGPS
use 1000 UNet evaluations in all reported settings. DPS-MAP uses 1000
respaced reverse steps on FFHQ and 100 steps on ImageNet. DMPS uses
100 steps on FFHQ and ImageNet Gaussian deblurring, and 1000 steps on
ImageNet motion and uniform deblurring.

\end{document}